\documentclass[10pt,twocolumn,letterpaper]{article}

\usepackage{wacv}
\usepackage{times}
\usepackage{epsfig}
\usepackage{graphicx}
\usepackage{amsmath}
\usepackage{amssymb}

\usepackage{authblk}
\usepackage{url}

\usepackage{times}  
\usepackage{helvet}  
\usepackage{courier}  
\usepackage{graphicx}  
\usepackage{amsmath}
\usepackage{amssymb}
\usepackage{siunitx}
\usepackage{subfig}
\usepackage[export]{adjustbox}
\usepackage{mathtools}
\usepackage{amsthm}
\newtheorem{theorem}{Theorem}[section]
\usepackage{wrapfig}
\usepackage{verbatim}

\usepackage[toc,page]{appendix}

\wacvfinalcopy 

\def\wacvPaperID{618} 

\ifwacvfinal\fi
\begin{document}

\title{Learning Generative Models of Tissue Organization with Supervised GANs}



\author[1]{Ligong Han}
\author[2]{Robert F. Murphy}
\author[1]{Deva Ramanan}
\affil[1]{Robotics Institute, Carnegie Mellon University}
\affil[2]{Computational Biology Department and Department of Biological Sciences, Carnegie Mellon University}

\maketitle
\ifwacvfinal\thispagestyle{empty}\fi

\begin{abstract}
A key step in understanding the spatial organization of cells and tissues is the ability to construct generative models that accurately reflect that organization. In this paper, we focus on building generative models of electron microscope (EM) images in which the positions of cell membranes and mitochondria have been densely annotated, and propose a two-stage procedure that produces realistic images using Generative Adversarial Networks (or GANs) in a supervised way. In the first stage, we synthesize a label ``image'' given a noise ``image'' as input, which then provides supervision for EM image synthesis in the second stage. The full model naturally generates label-image pairs. We show that accurate synthetic EM images are produced using assessment via (1) shape features and global statistics, (2) segmentation accuracies, and (3) user studies. We also demonstrate further improvements by enforcing a reconstruction loss on intermediate synthetic labels and thus unifying the two stages into one single end-to-end framework.
\end{abstract}


\section{Introduction} \label{sec:intro}
\begin{figure}
\centering
    \subfloat[Generative pipeline]{%
    \begin{minipage}{1\linewidth}
    \centering
    \includegraphics[width=0.95\linewidth]{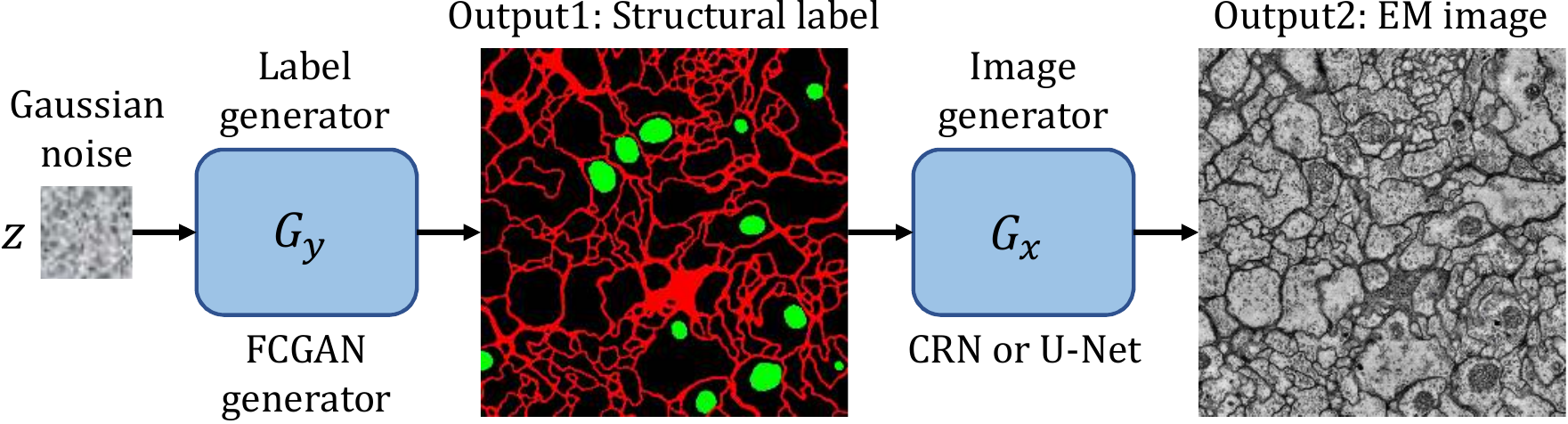}
    \end{minipage}%
    }\par
    \subfloat[Ground-truth]{%
    \begin{minipage}{0.33\linewidth}
    \centering
    \includegraphics[width=0.95\linewidth]{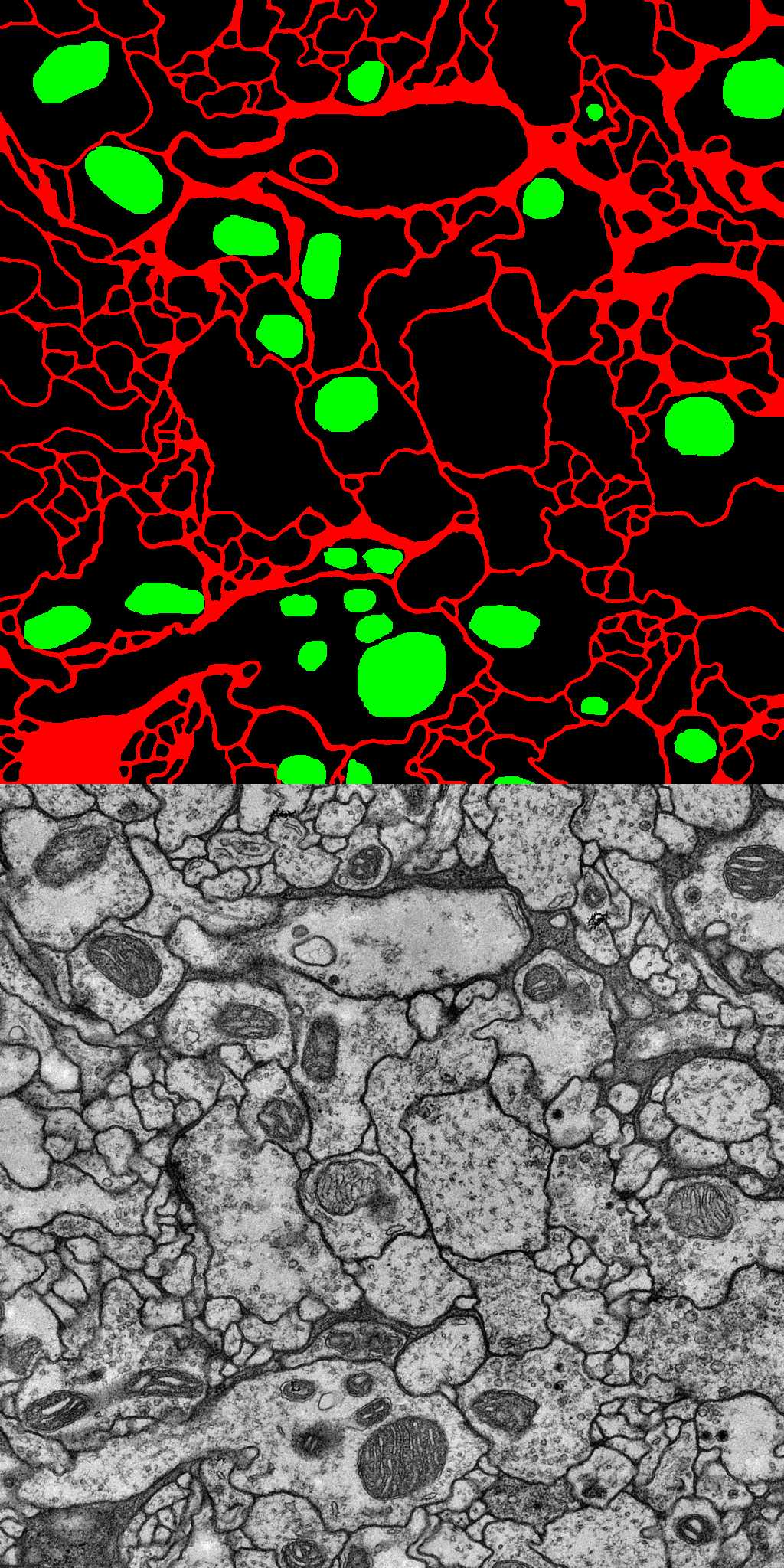}
    \end{minipage}%
    }
    \subfloat[SGAN (ours)]{%
    \begin{minipage}{0.33\linewidth}
    \centering
    \includegraphics[width=0.95\linewidth]{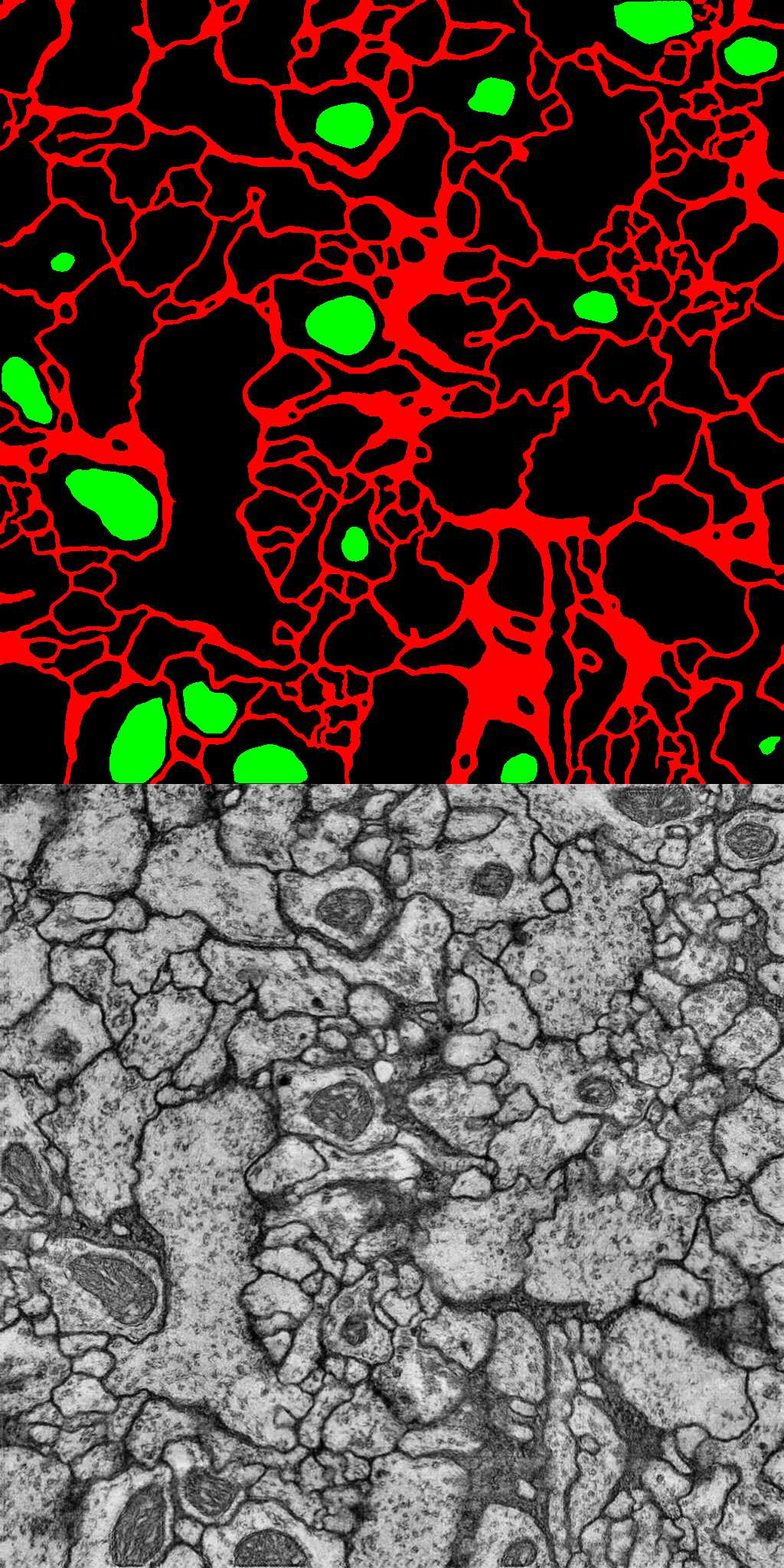}
    \end{minipage}%
    }
    \subfloat[UnsupervisedGAN]{%
    \begin{minipage}{0.33\linewidth}
    \centering
    \includegraphics[width=0.95\linewidth]{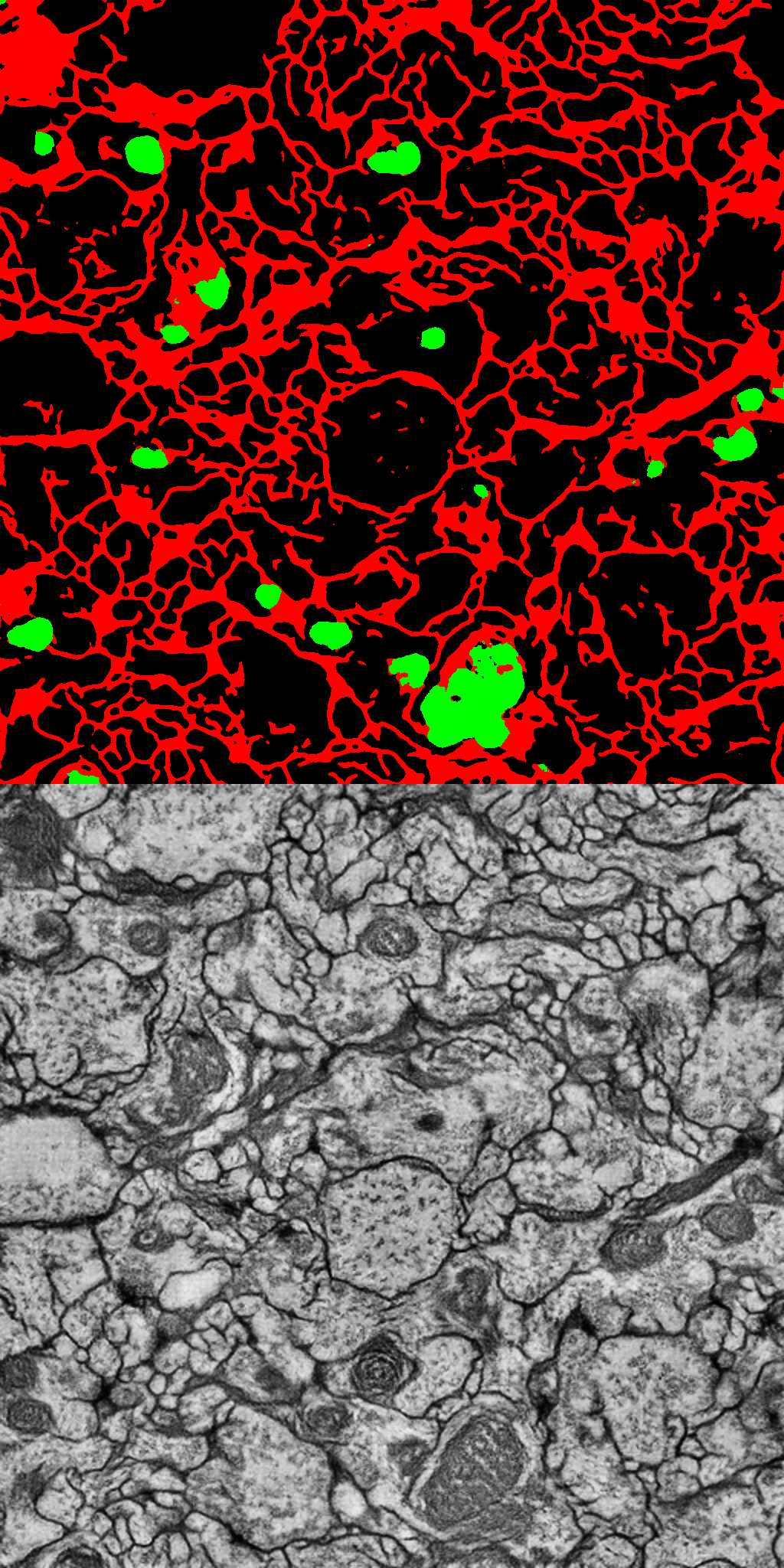}
    \end{minipage}%
    }
    \caption{(a) Generative pipeline: Given noise ``image'' $z$ sampled from a Gaussian distribution, our label generator $G_y$ generates a label image, which is then translated into an EM image by $G_x$. (b) Ground-truth label-image pair. (c) Label and image pair generated by our supervised GANs (SGAN), that is capable of generating continuous membranes (red lines) and correctly positioned mitochondria (green blobs). (d) Image synthesized by unsupervised GANs, in which the label is generated by a pre-trained semantic segmentation network. Unsupervised GAN is able to produce pixel-level details locally but fails to capture structures globally.}
    \label{fig:sample}
\end{figure}
Much research in the life sciences is now driven by large amounts of biological data acquired through high-resolution imaging~\cite{chessel2017overview,meijering2016imagining}. Such data represents an important application domain for automated machine vision analysis. Most past work has been {\em discriminative} in nature, focusing on trying to determine whether imaged samples differ between different patients, tissues, cell types or treatments~\cite{boland2001neural,
carpenter2006cellprofiler}.  A more recent focus has been on constructing {\em generative} models, especially of cells or tissues~\cite{zhao2007automated,svoboda2011generation}.  Such generative approaches are required in order to be able to combine spatial information on different cell types or cell organelles learned from separate images (and potentially different imaging modalities) into a single model. This is needed because of the difficulty of visualizing all components in a single image.  Images can be used to perform spatially-accurate simulations of cell or tissue biochemistry~\cite{loew2001virtual}, and synthetic images that combine many components can dramatically enhance the accuracy and usefulness of such simulations.

\noindent {\bf Microscopy imaging:} At the cellular scale, the dominant modes of imaging used are fluorescence microscopy (FM) and electron microscopy (EM). From the machine vision perspective, these methods differ dramatically in their resolution, noise, and the availability of labels for particular structures. FM works by tagging particular molecules or structures with fluorescence probes, adding a powerful form of sparse biological supervision to the captured images (which does not require human intervention). However, the spatial resolution of FM ranges from a limit of approximately 250 nm for traditional methods to 20-50 nm for super-resolution methods. By contrast, EM allows for significantly higher resolution (0.1-1 nm per pixel), but ability to automatically produce labels is limited and manual annotation can be very time-consuming. Analysis of EM images is also challenging because they contain lower signal-to-noise ratios than FM.

\noindent {\bf Our goal:} We wish to build holistic generative models of cellular structures visible in high-resolution microscopy images. In the following, we point out several unique aspects of our approach, compared to related work from both biology and machine learning.

\noindent {\bf Data:} 
We focus on EM images that contain enough resolution to view structures of interest. This in turns means that supervised labels (e.g., organelle segmentation masks) will be difficult to acquire. Indeed, it is quite common for standard EM benchmark datasets to contain only tens of images, illustrating the difficulty of acquiring human-annotated labels~\cite{arganda2015crowdsourcing}. Most work has focused on segmentation of individual cells or organelles within such images~\cite{jones2013neuron}. In contrast, we wish to build models of {\em multiple} cells and their internal organelles, which is particular challenging for brain tissue due to the overlapping meshwork of neuronal cells.

\noindent {\bf Generative models:} Past work on EM image analysis has focused on discriminative membrane detection~\cite{ciresan2012deep,lee2015recursive}. Here we seek a high-resolution generative model of cells and the spatial organization of their component structures. Generative models of cell organization have been a long sought-after goal~\cite{zhao2007automated,svoboda2009generation}, because at some level, such models are a required component of any behavioral cell model that depends on constituent proteins within organelles.

\noindent {\bf GANs:} First and foremost, we show that generative adversarial networks (GANs)~\cite{goodfellow2014generative} can be applied to build remarkably-accurate generative models of multiple cells and their structures, significantly outperforming prior models designed for FM images. To do so, we add three innovations to GANs: First, in order to synthesize large high-resolution images (similar to actual recorded EM images), we introduce {\em fully-convolutional} variants of GANs that exploit the spatial stationarity of cellular images. Note that such stationarity may not present in typical natural imagery (which might contain, for example, a characteristic horizon line that breaks translation invariance). Secondly,  in order to synthesize natural geometric structures across a variety of scales, we add {\em multi-scale discriminators} to guide the generator to produce images with realistic multi-scale statistics. Thirdly, and most crucially, we make use of {\em supervision} to guide the generative process to produce semantic structures (such as cell organelles) with realistic spatial layouts. Much of the recent interest in generative models (at least with respect to GANs) has focused on unsupervised learning. But in some respect, synthesis and supervision are orthogonal issues. We find that standard GANs do quite a good job of generating texture, but sometimes fail to capture global geometric structures. 
We demonstrate that by adding supervised structural labels into the generative process, one can synthesize considerably more accurate images than an unsupervised GAN. 

\noindent {\bf Evaluation:} A well-known difficulty of GANs is their evaluation. By far, the most common approach is qualitative evaluation of the generated images. Quantitative evaluation based on perplexity (the log likelihood of a validation set under the generative distribution) is notoriously difficult for GANs, since it requires approximate optimization techniques that are sensitive to regularization hyper-parameters~\cite{metz2016unrolled}. Other work has proposed statistical classifier tests that are sensitive to the choice of classifier~\cite{lopez2016revisiting}. In our work, we use our supervised GANs to generate image-labels pairs that can be used to train discriminative classifiers, yielding quantitative improvements in prediction accuracy. Moreover, we consider the literature on generative cell models and use previously proposed metrics for evaluating generative models, including the consistency of various shape feature statistics across real vs generated images, as well as the stability of discriminative classifiers (for semantic labeling)  across real vs generated data~\cite{zhao2007automated}. We also perform user studies to measure a user's ability to distinguish real versus generated images. Crucially, we compare to strong baselines for generative models, including established parametric shape-based models as well as non-parametric generative models that memorize the data.

\section{Related Work}
There is a large body of work on GANs. We review the most relevant work here. 

\noindent {\bf GANs:} Our network architecture is based on DCGAN~\cite{radford2015unsupervised}, which introduces convolutional network connections. We make several modifications suited for processing biological data, which tends to be high-resolution and encode spatial structures at multiple scales. As originally defined, the first layer is {\em not} convolutional since it processes an input noise ``vector''. We show that by making all network connections convolutional (by converting the noise vector to a noise ``image''), the entire generative model is convolutional. This in turns allows for efficient training (through learning on small convolutional crops) and high-quality image synthesis (through generation of larger noise images). We find that multi-scale modeling is crucial to synthesizing accurate spatial structures across varying scales. While past work has incorporated multi-scale cues into the generative process~\cite{denton2015deep}, we show that multiscale {\em discriminators} help further produce images with realistic multi-scale statistics. 

\noindent {\bf Supervision:} Most GANs work with unsupervised data, but there are variants that employ some form of auxiliary labels. Conditional GANs make use of labels to learn a GAN that synthesizes pixels conditioned on an label image~\cite{pix2pix} or image class label~\cite{odena2016conditional}, but we use supervision to learn an end-to-end generative model that synthesizes pixels given a noise sample. Similarly, methods for semi-supervised learning with GANs~\cite{zhang2017structured} tend to factorize generative process into disentangled factors similar to our labels. However, such factors tend to be global (such as an image class label), which are easier to synthesize than spatially-structured labels. From this perspective, our approach is similar to~\cite{wang2016generative}, who factorizes image synthesis into separate geometry and style stages. In our case, we make use of semantic labels rather than metric geometry as supervision. Finally, most related to us is~\cite{osokin2017gans}, who uses GANs to synthesize fluorescent images using implicit supervision from cellular staining. Our work focuses on EM images, which are high resolution (and so allows for modeling of more detailed substructure), and crucially makes use of semantic supervision to help guide the generative process.

\section{Supervised GANs} \label{sec:sgan}
\begin{figure*}[t!]
\centering
   \includegraphics[width=.7\linewidth]{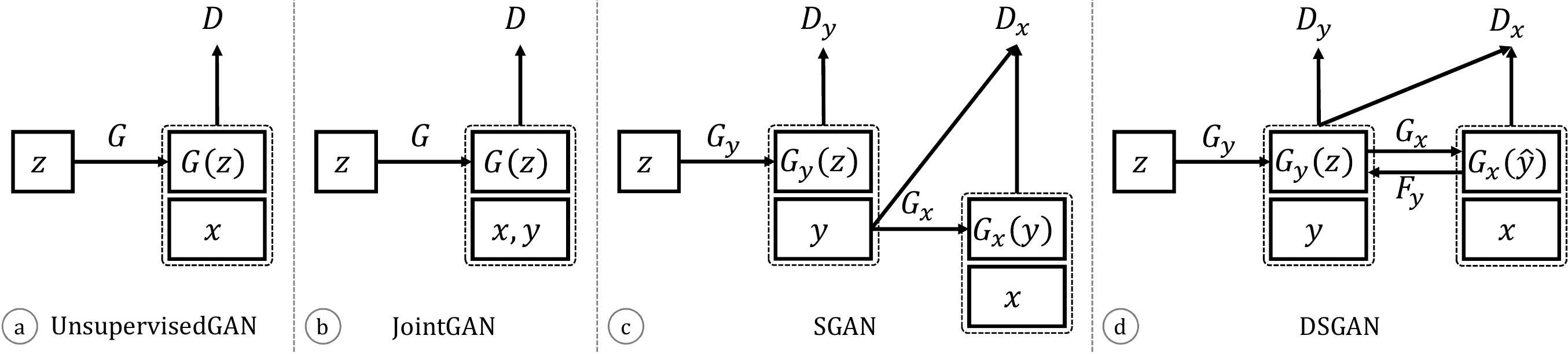}
\caption{We compare different GAN architectures for injecting supervision (provided with labels $y$) into a generative model of $x$. (a) A standard unsupervised GAN for generating $x$. (b) A GAN defined over a joint variable $x' = (x,y)$.  (c) SGAN, which is a supervised GAN that is composed of an initial GAN $\{G_y, D_y\}$ that generates labels $y$ followed by a conditional GAN $\{G_x, D_x\}$ that generates images $x$ from $y$. 
(d) A Deeply supervised GAN (DSGAN), equivalent to a single GAN that is provided deep supervision for generating labels at an intermediate stage. Performance is further improved by adding a reconstructor $F_y$ that ensures that generated images can be used to predict labels with low reconstruction error (Eq.~\ref{eq:dsgan_cgan_cycle}). }
\label{fig:sgan}
\end{figure*}

A standard GAN, originally proposed for unsupervised learning, can be formulated with a minimax value function $V(G, D)$:
\begin{align}
&\min_G \max_D V(G,D) \qquad \qquad \text{where} \nonumber\\
V(G, D) = &\mathbb{E}_{x\sim p_{x}}[\log(D(x))] + \mathbb{E}_{z\sim p_{z}}[\log(1-D(G(z)))] \nonumber
\end{align}
\vspace*{-27pt}
\begin{flalign}
\text{[\bf{UnsupervisedGAN}]} && \label{eq:gan}
\end{flalign}
\noindent and $x$ denotes image and $z$ denotes a latent noise vector. As defined, the minimax function can be optimized with samples from the marginal data distribution $p_x$ and thus no supervision is needed. As shown in Fig.~\ref{fig:sample}, this tends to accurately generate low-level textures but sometimes fails to capture global image structures.
Assume now that we have access to image labels $y$ that specify spatial structures of interest. Can we use these labels to train a better generator? Presumably the simplest approach is to define a ``classic'' GAN over a joint variable $x' = (x,y)$:
\begin{align}
V(G,D) = &\mathbb{E}_{x,y\sim p_{xy}}[\log(D({x,y}))] + \nonumber\\
\text{[\bf{JointGAN}]} \enskip \qquad &\mathbb{E}_{z\sim p_{z}}[\log(1-D(G(z)))] \qquad \quad \label{eq:jointgan} 
\end{align}
\noindent {\bf Factorization:} Rather than learning a generative model for the joint distribution over $x,y$, we can factorize it into $p(x,y) = p(y)p(x|y)$ and learn generative models for each factor. This factorization makes intuitive sense since it implicitly imposes a causal relation~\cite{kocaoglu2017causalgan}: first geometric labels are generated with $G_y: z \mapsto y$, and then image pixels are generated conditioned on the generated labels, $G_x: y \mapsto x$. We refer to this approach as {\em SGAN} (supervised GANs), as illustrated in Fig. \ref{fig:sgan}-b,c: 
\begin{align}
V(G, D) = &V_y(G_y,D_y) +  V_x(G_x,D_x) \quad \text{where} \nonumber \\
V_y(G_y,D_y) = & \mathbb{E}_{y\sim p_{y}}[\log(D_y(y))] + \nonumber\\
   & \mathbb{E}_{z\sim p_{z}}[\log(1-D_y(G_y(z)))] \quad \text{and} \nonumber\\
V_x(G_x,D_x) = & \mathbb{E}_{x,y\sim p_{xy}}[\log(D_x(x,y))] + \nonumber\\
\text{[\bf{SGAN}]} \qquad \qquad & \mathbb{E}_{y\sim p_{y}}[\log(1-D_x(G_x(y),y))]. \quad \label{eq:sgan_cgan} 
\end{align}
In theory, one could also factorize the joint into $p(x,y) = p(x) p(y|x)$, which is equivalent to training a standard unsupervised GAN for $x$ and a conditional model for generating labels from $x$. The latter can be thought of as a semantic segmentation network. We compare to such an alternative factorization in our experiments, and show that conditioning on labels first produces significantly more accurate samples of $p(x,y)$.

\noindent {\bf Optimization:} Because value function $V(G, D)$ decouples, one can train $\{D_y, G_y\}$ and $\{D_x, G_x\}$ independently: 
\begin{align}
\min_G \max_D V(G,D) =& \min_{G_y}\max_{D_y} V_y(G_y,D_y) + \nonumber\\
& \min_{G_x}\max_{D_x} V_x(G_x,D_x) \label{eq:sgan_decouple}
\end{align}

Using arguments similar to those from \cite{goodfellow2014generative}, one can show that SGAN can recover true data distribution where the discriminator $D$ and generators $G$ are optimally trained:
\begin{theorem}
\label{thm:optimality}
The global minimum of $C(G) = \max_D V(G,D)$ is achieved if and only if $q(y) = p(y)$ and $q(x | y) = p(x | y)$, where $p$'s are true data distributions and $q$'s are distributions induced by $G$.
\end{theorem}
\begin{proof}\renewcommand{\qedsymbol}{}
Given in Supplementary 1.
\end{proof}
\noindent {\bf End-to-end learning:} The above theorem demonstrates that SGANs will capture the true joint distribution over labels and data if trained optimally. However, when not optimally trained (because of optimization challenges or limited capacity in the networks), one may obtain better results through end-to-end training. Intuitively, end-to-end training optimizes $G_x(y)$ on samples of labels ${\hat y}$ produced by the initial generator $G_y(z)$, rather than ground-truth labels $y$. To formalize this, one can regard $G_y$ and $G_x$ as sub-networks of a {\em single} larger generator which is provided deep supervision at early layers: 
\begin{align}
V_x(G_x,D_x) = & \mathbb{E}_{x,y\sim p_{xy}}[\log(D_x(x,y))] + \label{eq:dsgan_cgan}\\
   & \mathbb{E}_{z\sim p_{z}}[\log(1-D_x(G_x(G_y(z)),G_y(z)))]. \nonumber
\end{align}
However in practice, samples from an imperfect $G_y$ makes it even harder to train $G_x$. Indeed, we observe that $G_x$ produces poor results when synthetic training labels are introduced. One possible reason is that discriminator $D_x$ will be focused on the differences between the real and predicted labels rather than correlations between labels and images. (Please refer to section~\ref{sec:exp} and Supplementary 3 for more analysis.) To avoid this, we force the generator to learn such correlations by also learning a {\em reconstructor}  $F_y(x): x \mapsto y$ that re-generates labels from images. We add a reconstruction loss $\mathcal{L}$ (similar to a ``cycle GAN''~\cite{zhu2017unpaired}) that ensures that $G_x$ will produce an image from which an accurate label can be reconstructed. We refer to this approach as a $DSGAN$ (Deeply Supervised GAN):
\begin{align}
\min_{G, F} \max_D V_y&(G_y,D_y) + V_x(G_x,F_y,D_x) \qquad \text{where} \nonumber\\
V_x(G_x,F_y,D_x) = & \mathbb{E}_{x\sim p_{xy}}[\log(D_x(x,y))] + \nonumber\\
   & \mathbb{E}_{y\sim p_{y}}[\log(1-D_x(G_x(y),y))] + \nonumber\\
   & \lambda_{reg} \mathbb{E}_{x,y\sim p_{xy}}[\mathcal{L}(y,F_y(x))] + \nonumber\\
   & \lambda_{cyc} \mathbb{E}_{y\sim p_{y}}[\mathcal{L}(y,F_y(G_x(y)))] + \nonumber\\
   & \lambda_{cyc} \mathbb{E}_{z\sim p_{z}}[\mathcal{L}(G_y(z),F_y(G_x(G_y(z))))].\nonumber
\end{align}
\vspace*{-25pt}
\begin{flalign}
\text{[\bf{DSGAN}]} && \label{eq:dsgan_cgan_cycle}
\end{flalign}
The above training strategy is reminiscent of ``teacher forcing''~\cite{williams1989learning}, a widely-used technique for learning recurrent networks whereby previous predictions of a network are replaced with their ground-truth values (in our case, replacing $G_y(z)$ with $y$). The same optimality condition as in Theorem~\ref{thm:optimality} also holds for DSGANs.

\noindent {\bf Label editing:} Another advantage of SGAN or DSGAN is label editing, because editing in label space is much easier than in image space. This allows us to easily incorporate human priors into the generating process. For example, at test time, we can perform image processing on synthetic labels such as to remove discontinuous membranes or to remove mitochondria that are concave or replace with its convex hull.

\noindent {\bf Conditional label synthesis:} We can further split labels into $y = (y_1, y_2)$. This allows us to learn explicit conditionals that might be useful for simulation (e.g., synthesizing mitochondria given real cell membranes). This may be suggestive of interventions in a causal model (CausalGAN~\cite{kocaoglu2017causalgan}).

\section{Network Architectures}
In this section we outline our GAN network architectures, focusing on modifications that allow them to scale to high-resolution multi-scale biological images. Specifically, we first propose a fully-convolutional label generator which allows arbitrary output sizes; then we describe a novel multi-scale patch-discriminator to guide the generators to produce images with realistic multi-scale statistics. The fully-convolutional generator and multi-scale discriminators define a {\em fully-convolutional GAN} (FCGAN).

\noindent {\bf Fully-convolutional generator:} 
Since the shape of both membrane and mitochondria are invariant to spatial location, a fully convolutional network is desirable to model the generators. Generators in previous works such as DCGAN take a noise vector as input. As a result, the size of their output images is predefined by their network architecture, thus unable to produce arbitrarily sized images at test time. We therefore propose to feed a noise ``image'' instead of a vector into the generator. The noise ``image'' is essentially a 3D tensor with the first two dimensions corresponding to the spatial positions. As illustrated in Fig.~\ref{fig:fcgan_G}, to synthesize arbitrarily large labels, we only need to modify the spatial size of the input noise. Fully-convolutional generator is an instantiation of the label generator in Fig.~\ref{fig:sample}-a.
\begin{figure}
\centering
\subfloat{\includegraphics[scale=0.4]{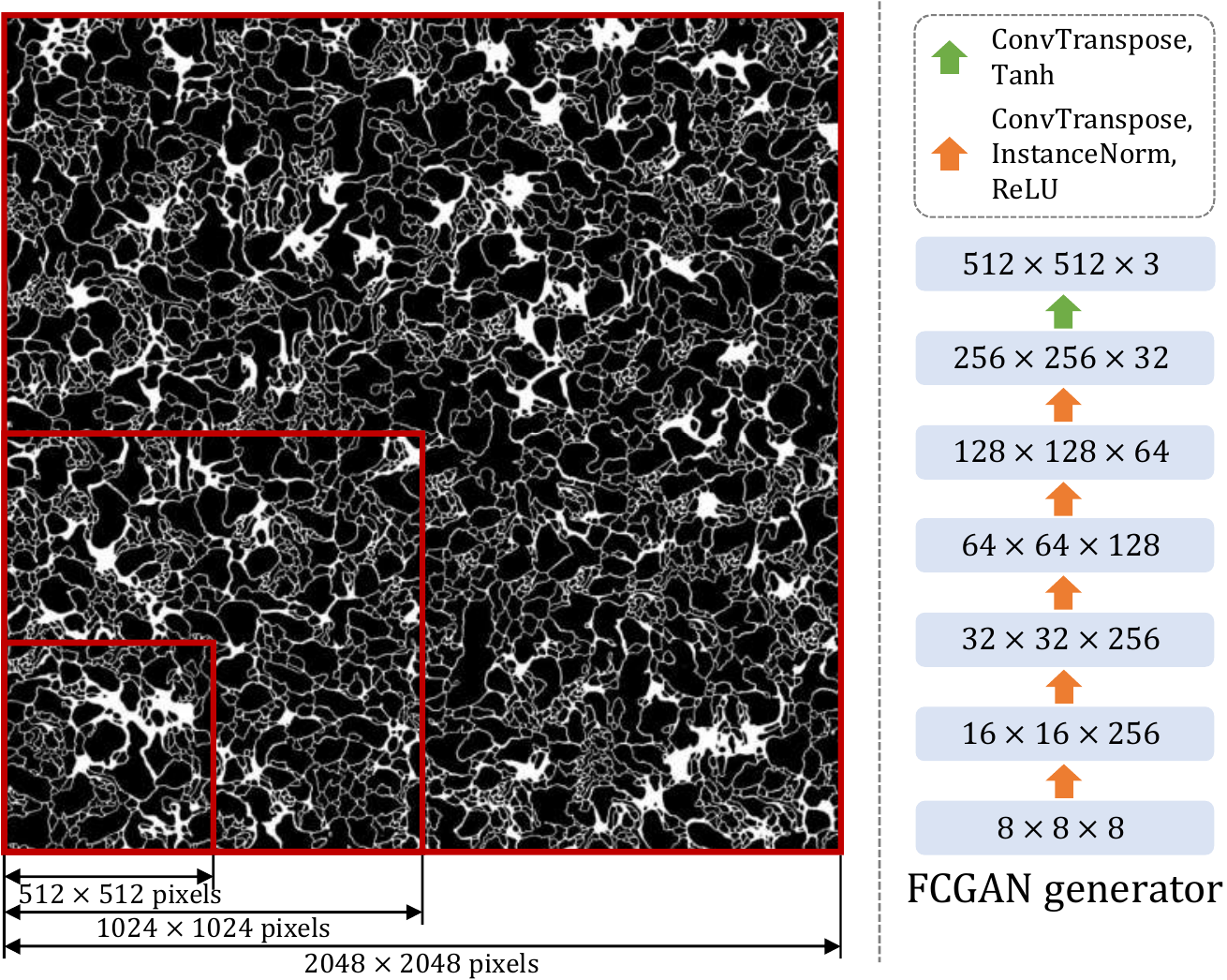}}
\caption{A fully-convolutional generator $G_y$ (FCGAN). Left: By changing the size of the input noise ``image'', our FCGAN generator can synthesize arbitrarily large labels. Right: Architecture of the fully-convolutional generator.}
\label{fig:fcgan_G}
\end{figure}

\noindent {\bf Multi-scale discriminator:} 
\begin{figure}
\centering
\includegraphics[scale=0.4]{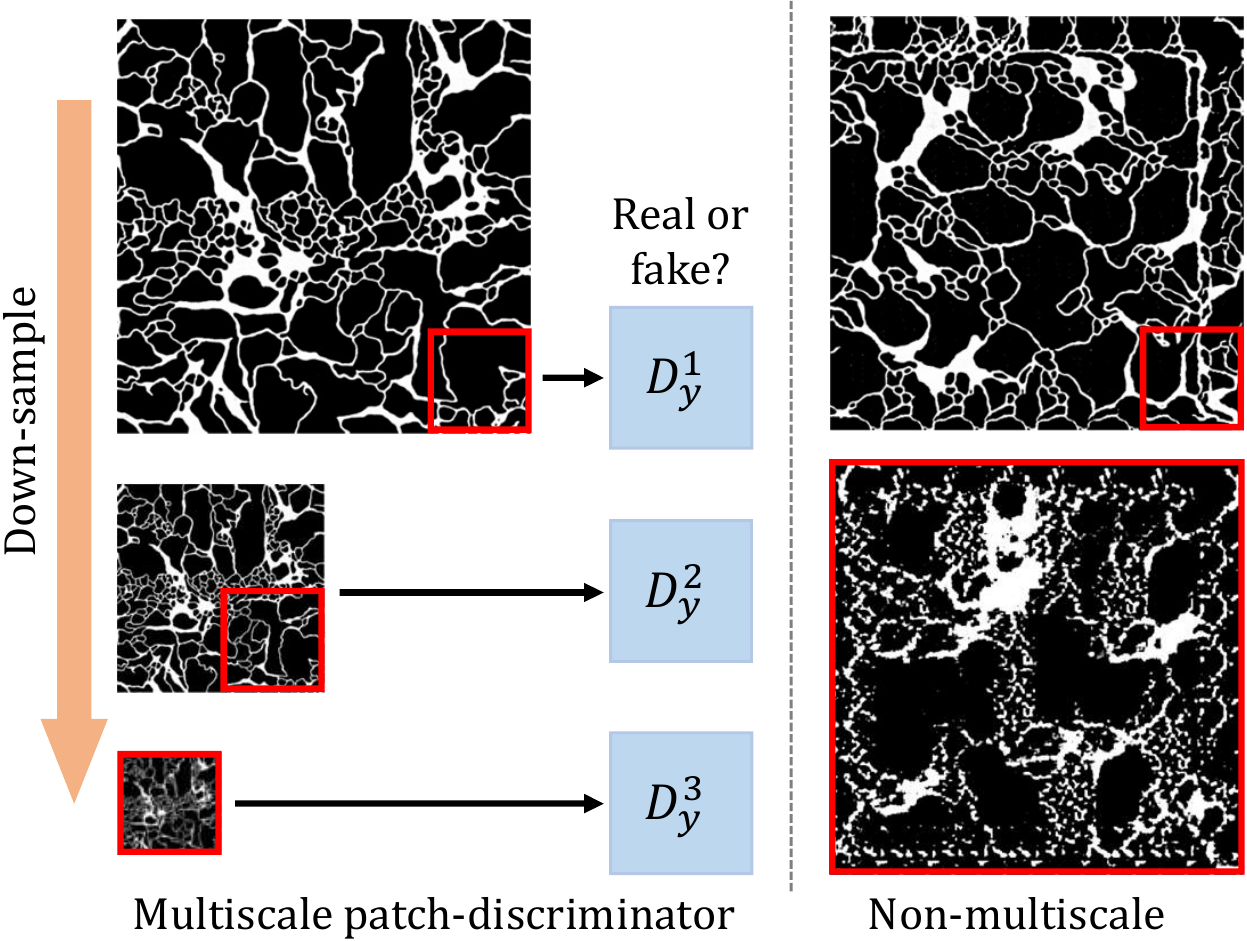}
\caption{Multi-scale discriminators for $D_y$ and $D_x$. Left: We construct an image pyramid from the generated label (or image), and feed patches from this pyramid to multiple patch-based discriminators. Right: Single-scale discriminators with small receptive fields (top) tend to produce accurate local structure, but inaccurate repetitive global structure. Similarly, single-scale discriminators with large receptive fields produce accurate global structure, but fail to generate accurate local textures.}
\label{fig:fcgan_D}
\end{figure}
We initially experiment with the patch-based discriminator network as in~\cite{pix2pix}, and find that the quality of synthesized labels relates to the patch size chosen for the discriminator. On the one hand, if we use a small patch size, the synthesized label has locally realistic patterns, but the global structure is wrong as it contains repetitive patterns (see Fig.~\ref{fig:fcgan_D} top-right). On the other hand, if we use a large patch-size, the output label image resembles a roughly plausible global structure, but lacks local details (see Fig.~\ref{fig:fcgan_D} bottom-right).

To ensure the generators produce both globally and locally accurate labels and images, we propose a multi-scale discriminator architecture. As illustrated in Fig.~\ref{fig:fcgan_D}, the input label (or image) is first down-sampled to different scales and then fed into individual discriminators. The final discriminator output is a weighted summation of the discriminators for each scales:
\begin{align}
V_y(G_y,D_y) = & \mathbb{E}_{y\sim p_{y}}[\sum_{i \in I} \lambda_i \log(D^i_y(\pi_i(y)))] + \label{eq:fcgan_ms}\\
   & \mathbb{E}_{z\sim p_{z}}[\sum_{i \in I} \lambda_i \log(1-D^i_y(\pi_i(G_y(z))))] \nonumber
\end{align}
Here, $i$ is the image pyramid level index, $\pi_i$'s denote down-sample transformations and $\lambda_i$'s are predefined coefficients.

\noindent {\bf Conditional generator:} Inspired by cascaded refinement networks (CRN) from~\cite{chen2017photographic}, we design architecturally similar generators for both {\em conditional image synthesis} ($y \rightarrow x$) and {\em conditional label synthesis} ($y_1 \rightarrow y_2$). Compared to U-Net~\cite{ronneberger2015u} which is originally adopted in pix2pix, CRN is less prone to mode collapse. Please see more discussions in Supplementary 2.
\begin{figure*}
\centering
\includegraphics[width=0.8\linewidth]{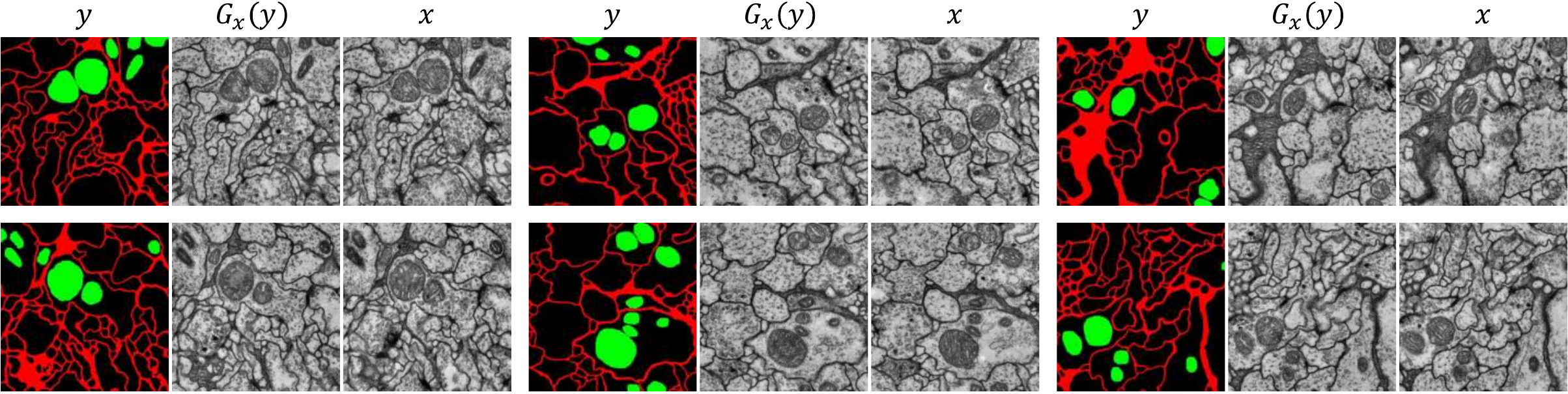}
\caption{Conditional image synthesis. Given true label $y$, we sample image $G_x(y)$, compared to real image $x$.}
\label{fig:gx}
\end{figure*}
\begin{figure*}
\begin{centering}
    \subfloat[Synthesized label-image pairs by SGAN.]{%
    \begin{minipage}{0.8\linewidth}
    \centering
    \includegraphics[width=1\linewidth]{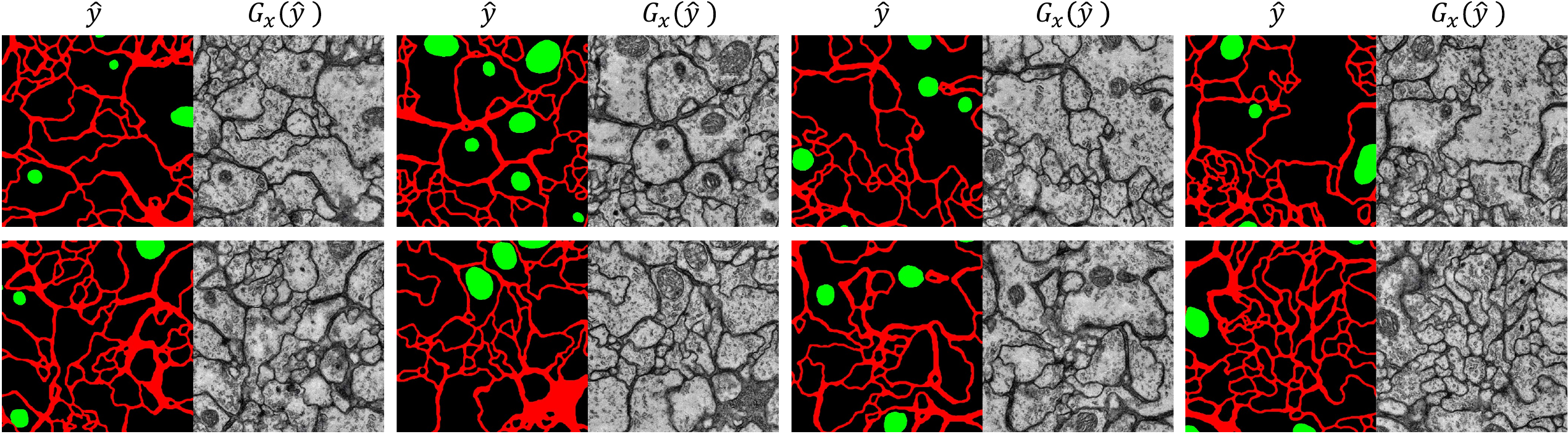}
    \end{minipage}%
    }\par
    \subfloat[Synthesized label-image pairs by JointGAN.]{%
    \begin{minipage}{0.4\linewidth}
    \centering
    \includegraphics[width=1\linewidth]{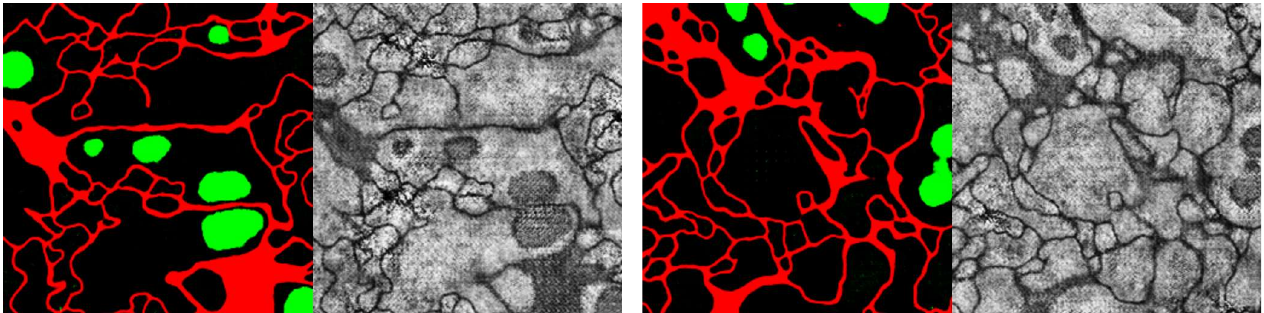}
    \end{minipage}%
    }
    \hspace{5pt}
    \subfloat[Synthesized label-image pairs by UnsupervisedGAN.]{%
    \begin{minipage}{0.4\linewidth}
    \centering
    \includegraphics[width=1\linewidth]{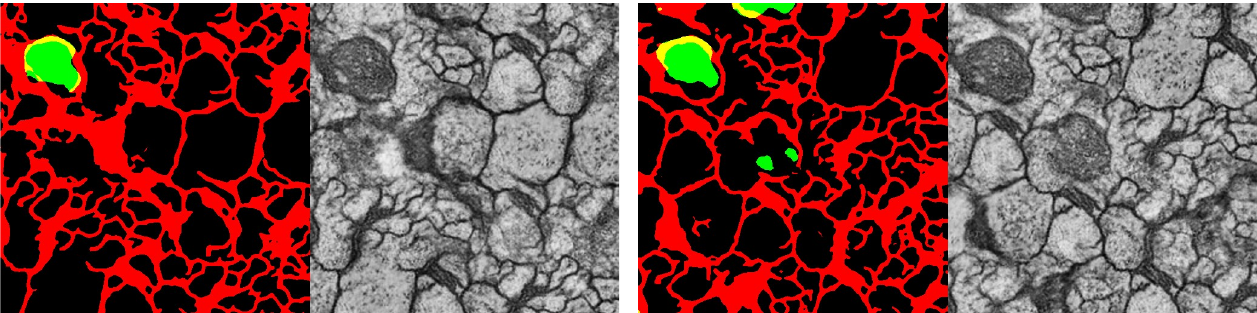}
    \end{minipage}%
    }\par
    \caption{Samples of our full model. (a) We first sample synthetic label $\hat{y}=G_y(z)$ from noise $z$, then generate image $G_x(\hat{y})$. We perform label editing (remove discontinuous membranes and concave mitochondria) on synthetic labels. (b) Label-image pairs directly synthesized by FCGAN. (c) Images are first generated by FCGAN then labels are inferred by an off-the-shelf segmentation network. Some pixels are labeled in yellow because because we use two separate segmentation networks for membranes and mitochondria.}
    \label{fig:sgan_sample}
\end{centering}
\end{figure*}

\section{Experiments} \label{sec:exp}
As illustrated in Fig.~\ref{fig:sample}-a, the proposed generative process contains two parts: (1) noise $\rightarrow$ label, (2) label $\rightarrow$ image. Thus our generative models output both labels and images that are paired. In this paper, we also evaluate our methods on these two levels: (1) labels, and (2) images. Particularly, on the label level we locally compare the {\em shape features} of single cells with real ones, and globally we compute {\em statistics} of multiple cells. On labels, we also evaluate the {\em model capacity}. On the image level, we measure image qualities by {\em segmentation accuracy}. Also, {\em user studies} are conducted on both levels.

\subsection{Metrics and Baselines} \label{sec:metric_baseline}
\noindent {\bf Shape features:} Following past work~\cite{zhao2007automated}, we evaluate the accuracy of synthetic images by (1) training a real/fake classifier, and (2) counting the portion of synthetic samples that fool the classifier. We train SVM classifiers on a set of 89  features~\cite{zhao2007automated} that have been demonstrated to very accurately distinguish cell patterns in FM images, and which are extracted from {\em label images} of single cells with mitochondria. Example statistics include 49 Zernike moment features, 8 morphological features, 5 edge features, 3 convex hull features \etc. 

\noindent {\bf Global statistics:} Such shape-based features used above are typically defined for a single cell. We therefore also extracted global statistics across multiple cells, including distributions of cell size, mitochondria size and mitochondria roundness~\cite{zheng2015traditional} \etc.

\noindent {\bf User studies:} We design an interface similar to that in~\cite{salimans2016improved}, where generated images are presented with a prior of $50\%$. Intermediate labels are edited for better visual quality (samples shown in Fig.~\ref{fig:sgan_sample}, cropped to $512 \times 512$).



\noindent {\bf Dataset:} We used a publicly available VNC dataset~\cite{vncdataset} that contains a stack of 20 annotated sections of the Drosophila melanogaster third instar larva ventral nerve cord (VNC) captured by serial section Transmission Electron Microscopy (ssTEM). The spatial resolution is $4.6 \times 4.6 \times 50~\text{nm/pixel}$. It provides segmentation annotations for cell membranes, glia, mitochondria and synapse. Through out experiments, the first 10 sections are used for training and the remaining 10 sections are used for validation.

\noindent {\bf Parametric baseline:} To construct baselines for our proposed methods, we compare to a well-established {\em parametric model} from~\cite{zhao2007automated} for synthesizing fluorescent images of single cells, which is trained by substituting nuclei with mitochondria. One noticeable disadvantage of this approach is that it lacks the capability to synthesize multiple cells. It also leverages supervision since the shape model is trained on labels.

\noindent {\bf Non-parametric baseline:} We also consider a non-parametric baseline that generates samples by simple selecting a random image from the training set. We know that, by construction, such a ``trivial'' generative model will produce perfectly realistic samples, but importantly, those samples will not generalize beyond the training set. We demonstrate that such a generative model, while quite straightforward, presents a challenging baseline according to many evaluation measures.

\subsection{Results} \label{sec:result}
\begin{figure}
  \begin{centering}
      \subfloat[Non-Parametric]{%
      \begin{minipage}{0.4\linewidth}
      \centering
      \includegraphics[width=0.98\linewidth]{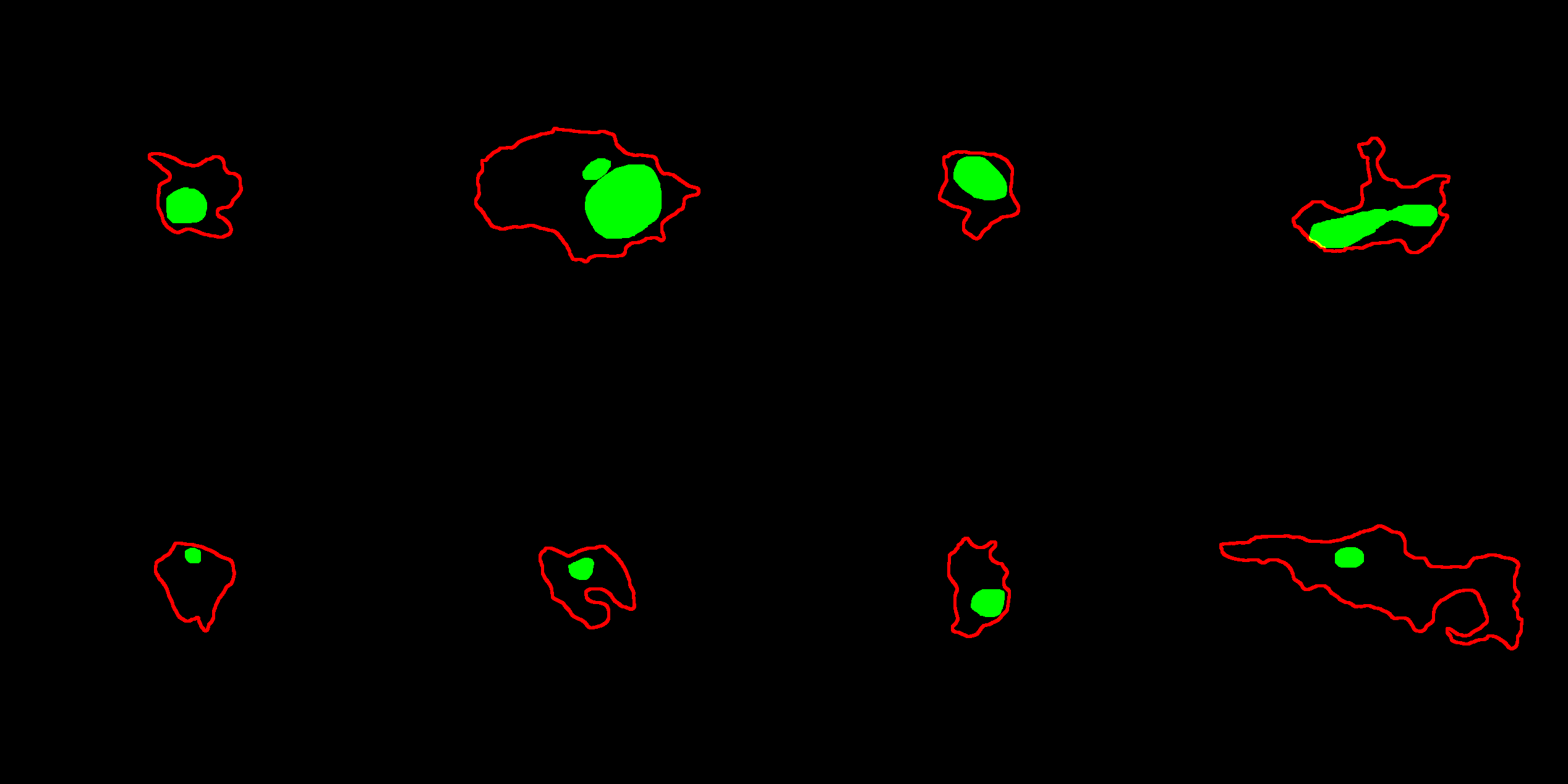}
      \end{minipage}%
      }
      \subfloat[Parametric]{%
      \begin{minipage}{0.4\linewidth}
      \centering
      \includegraphics[width=0.98\linewidth]{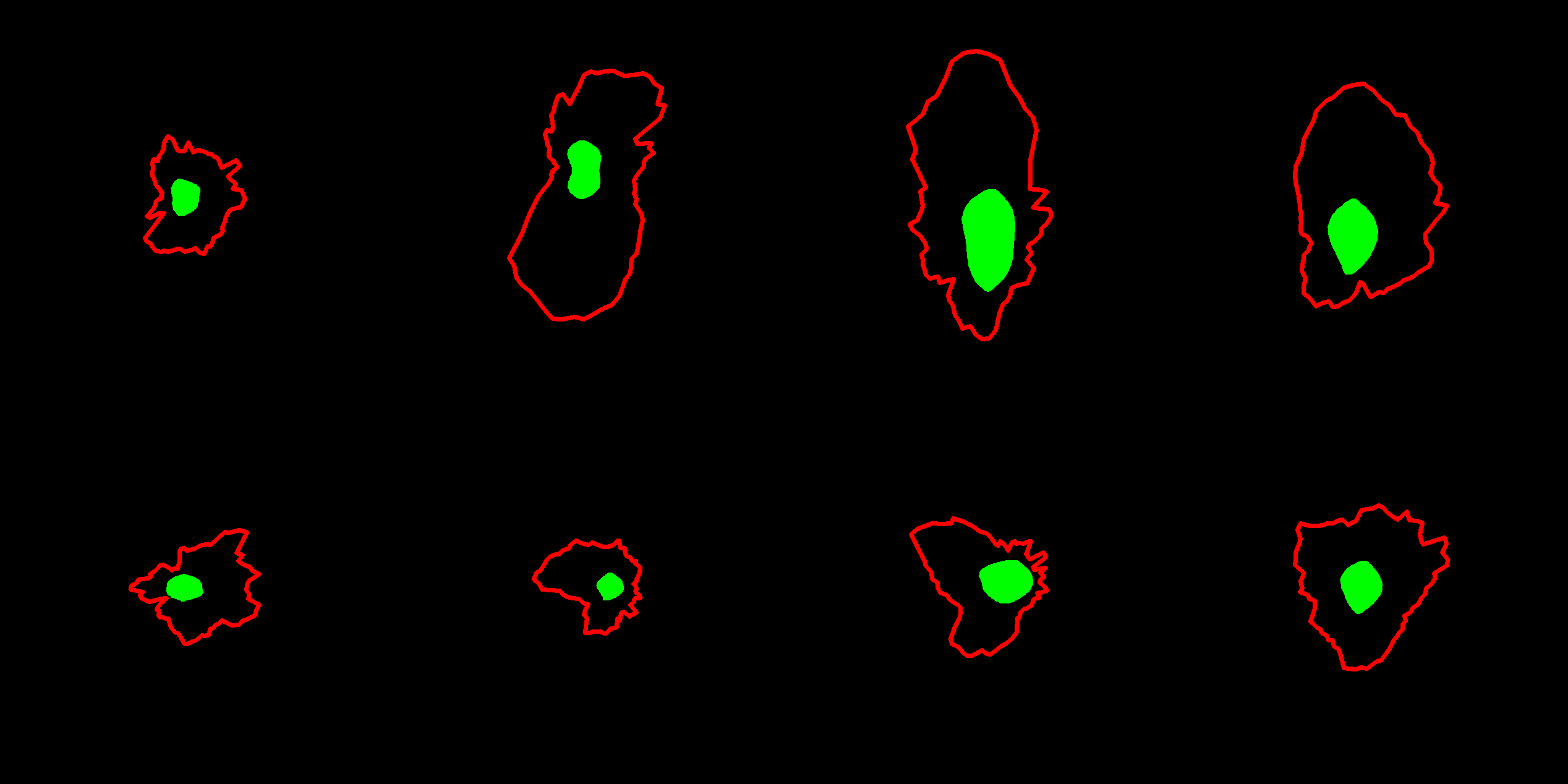}
      \end{minipage}%
      }\par
      \subfloat[SGAN]{%
      \begin{minipage}{0.4\linewidth}
      \centering
      \includegraphics[width=0.98\linewidth]{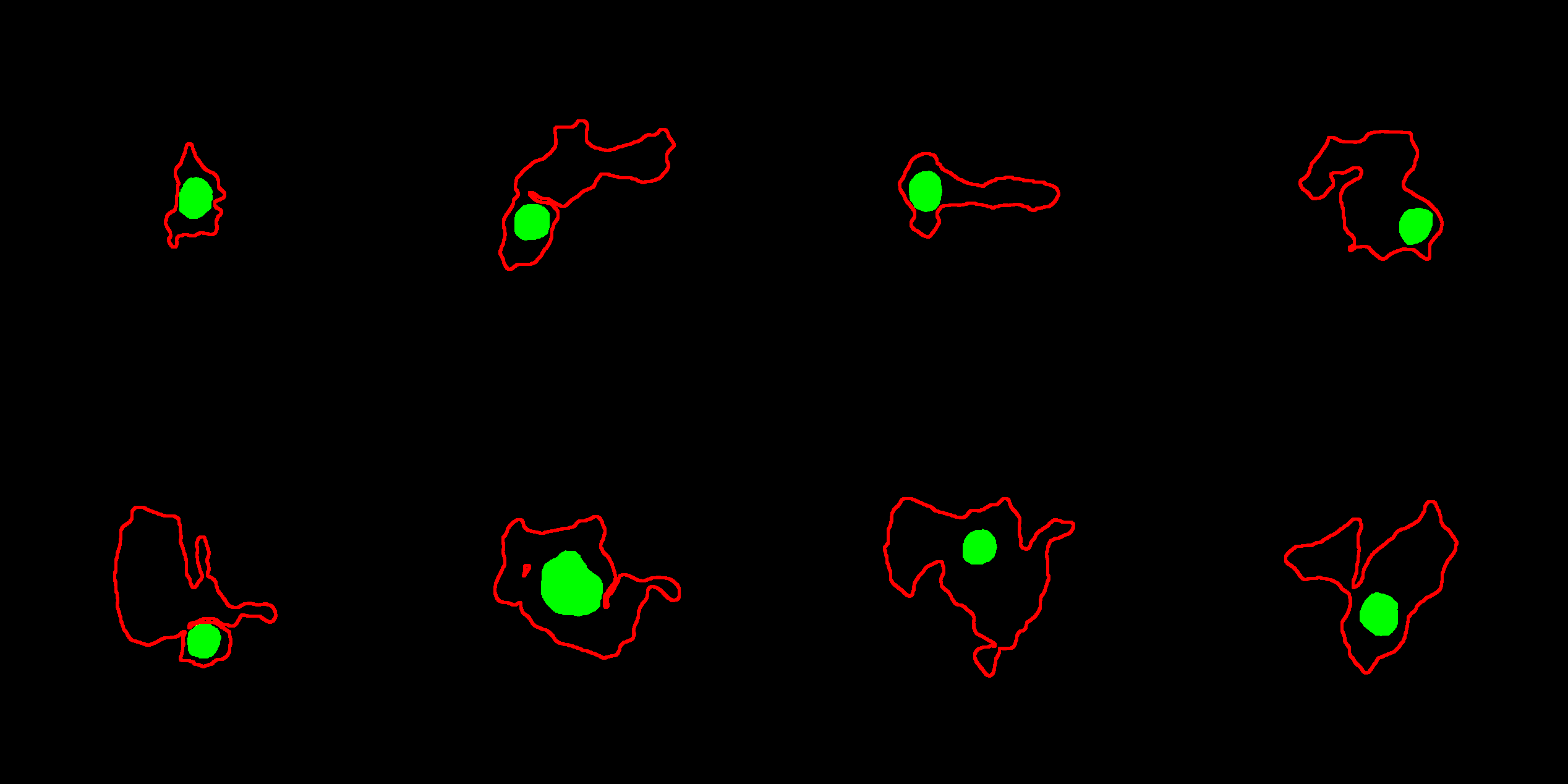}
      \end{minipage}%
      }
      \subfloat[DSGAN]{%
      \begin{minipage}{0.4\linewidth}
      \centering
      \includegraphics[width=0.98\linewidth]{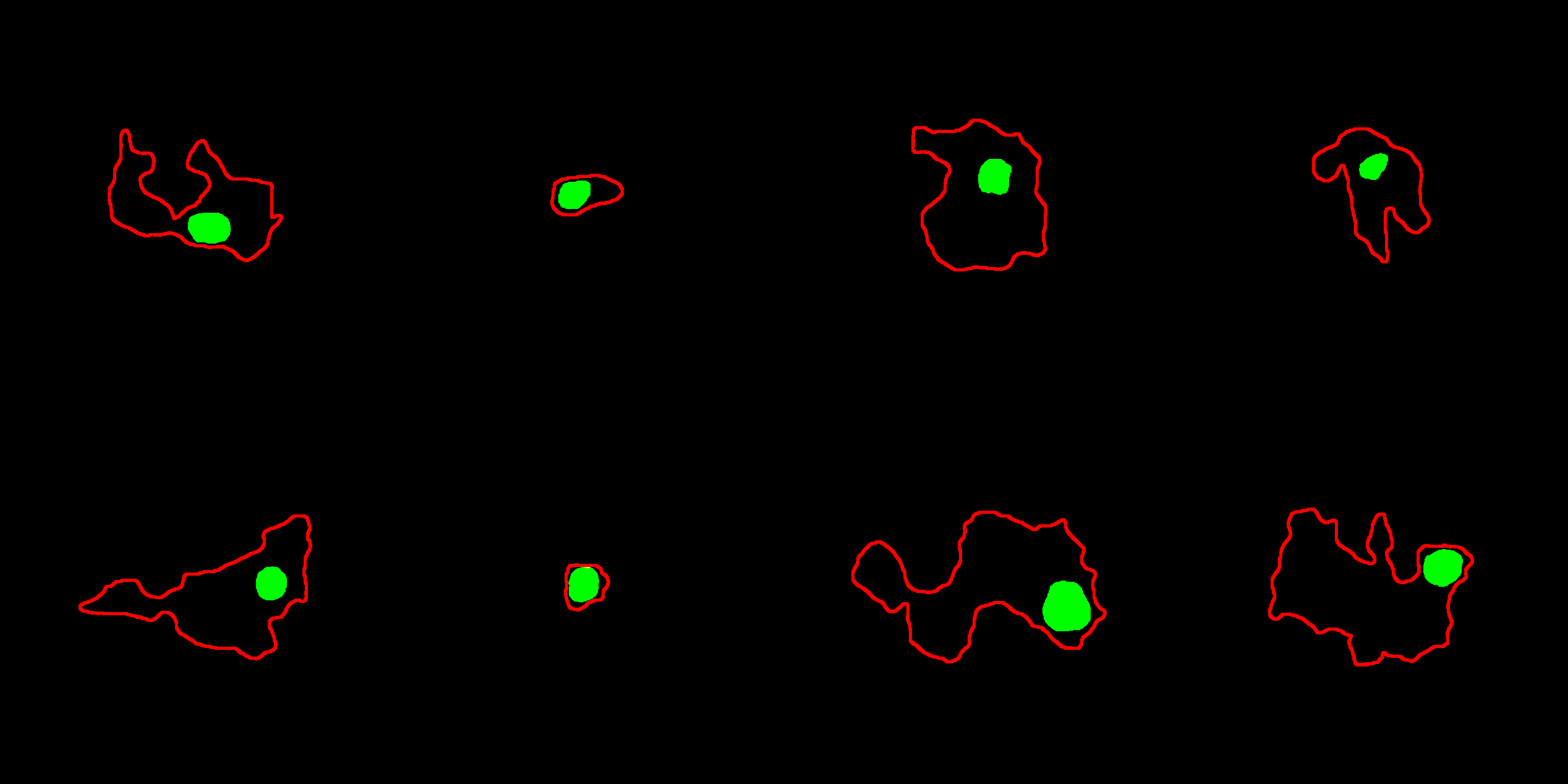}
      \end{minipage}%
      }\par
      \caption{Samples of single cell labels, from on which biological markers are extracted. More samples for other baselines are given in Supplementary 5. Quantitative results can be found in Fig.~\ref{fig:acc} (a) (b).}
      \label{fig:single}
  \end{centering}
\end{figure}
\begin{figure}
  \begin{center}
    \subfloat[Comparison of supervision strategies]{\includegraphics[scale=0.45, valign=t]{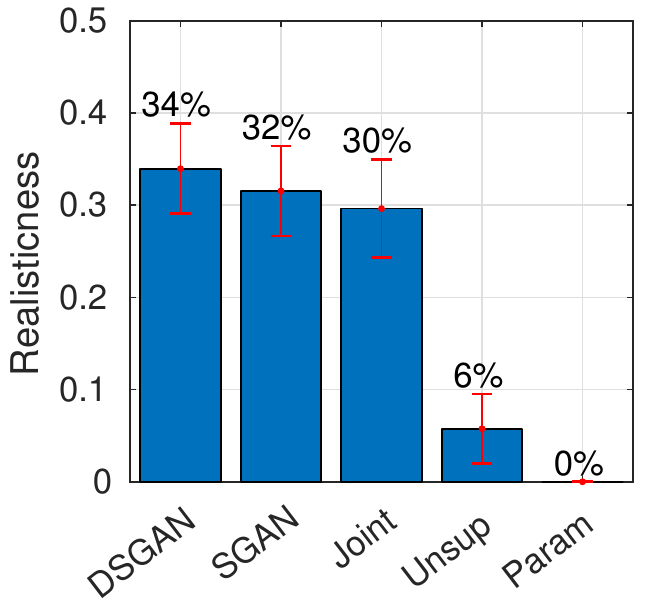}}
    \hspace{10pt}
    \subfloat[Comparison of network architectures]{\includegraphics[scale=0.45, valign=t]{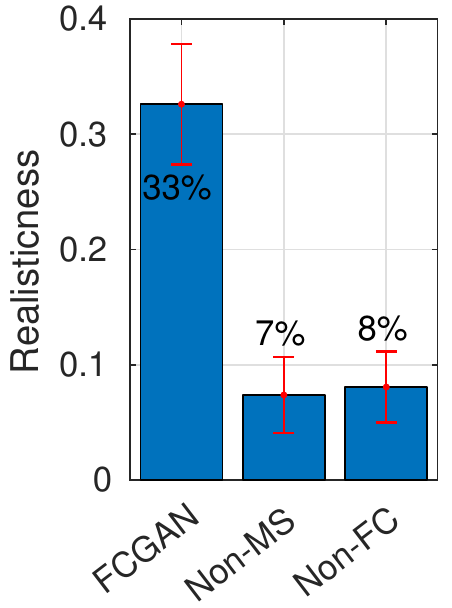}}
    \caption{Evaluating realism with image classifiers. SVM classifiers are trained to distinguish real and synthetic single-cell labels based on shape features as described in section~\ref{sec:metric_baseline}. Bar plot shows percentage of synthetic labels being classified as real ones, higher is better.}
    \label{fig:acc}
  \end{center}
\end{figure}
\begin{figure}
  \begin{center}
    \subfloat[User study on synthetic single-cell labels]{\includegraphics[scale=0.45, valign=t]{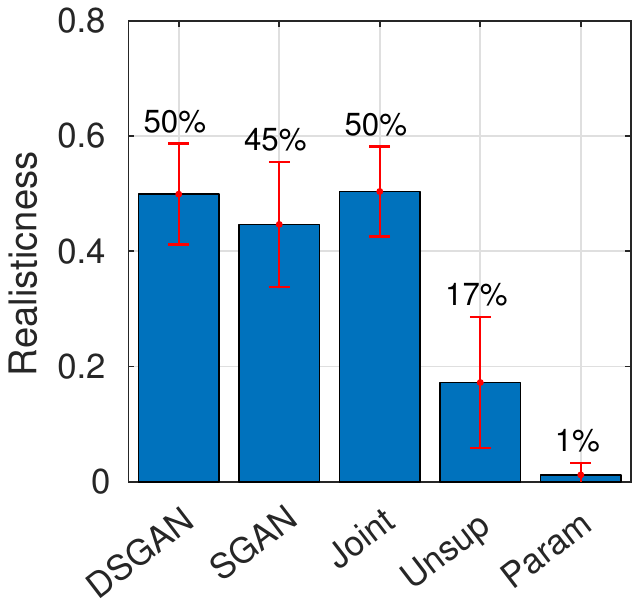}}
    \hspace{10pt}
    \subfloat[User study on synthetic images]{\includegraphics[scale=0.45, valign=t]{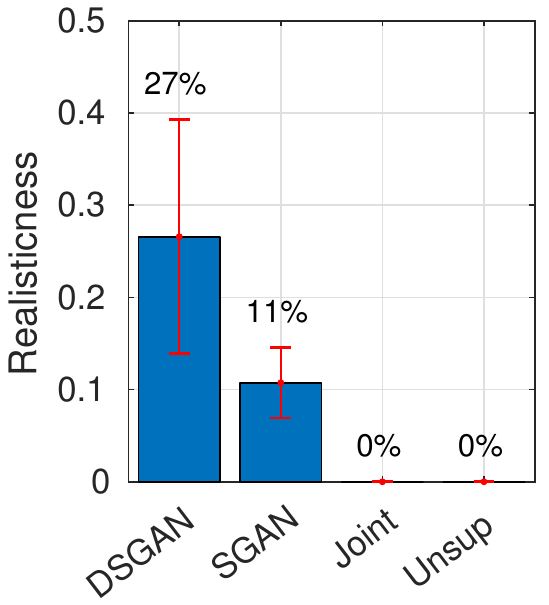}}
    \caption{Evaluating realism with user studies. Bar plot shows percentage of synthetic images being classified by users as real ones, higher is better. (a) Users are shown mixtures of real and synthetic labels of single cells. (b) Users are shown mixtures of real and synthetic full images.}
    \label{fig:user}
  \end{center}
\end{figure}
\begin{figure}
  \begin{center}
    \subfloat[Comparison of supervision strategies]{\includegraphics[scale=0.3]{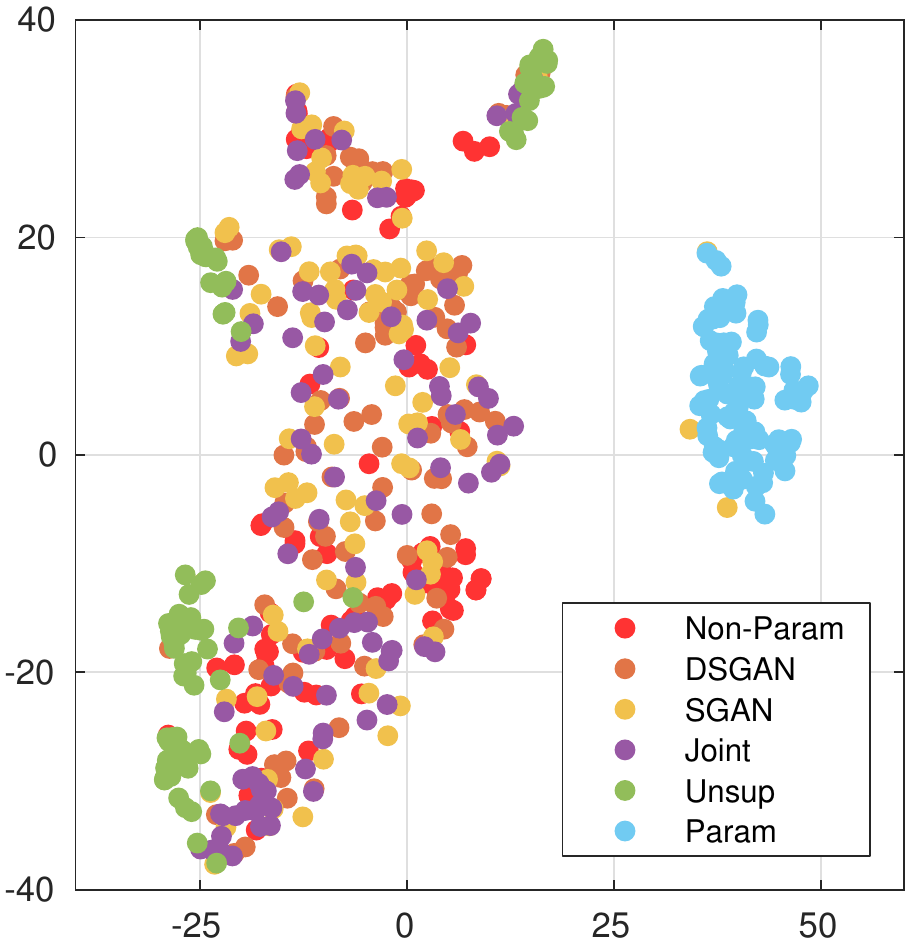}}
    \hspace{10pt}
    \subfloat[Comparison of network architectures]{\includegraphics[scale=0.3]{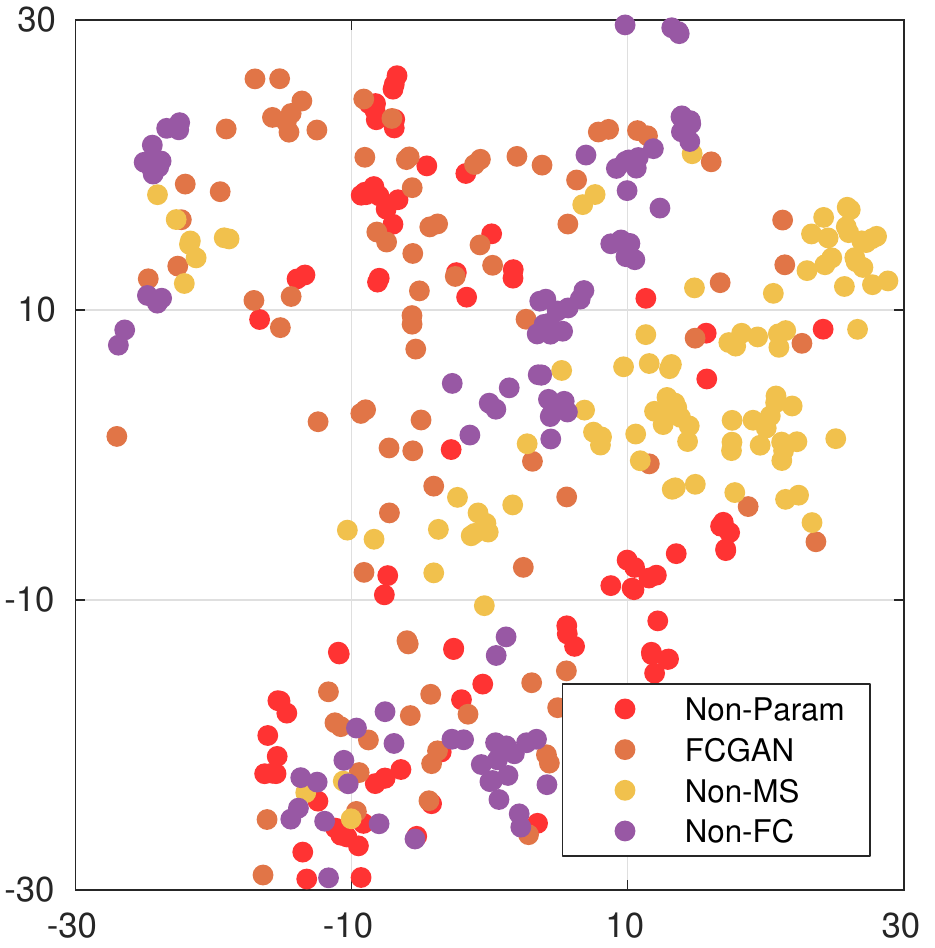}}
    \caption{2-D t-SNE~\cite{maaten2008visualizing} visualization of the 89-dimensional shape features. (a) Features of DSGAN, SGAN and JointGAN well overlap with real ones (Non-Parametric baseline), while features of UnsupervisedGAN or parametric baseline are easily separable. (b) Non-FC and Non-MS only covers parts of the real (projected) feature distribution.}
    \label{fig:tsne}
  \end{center}
\end{figure}
\begin{table*}
  \centering
  \resizebox{0.65\linewidth}{!}{
  \begin{tabular}{lcccc|ccc}
  \hline
  & \begin{tabular}{@{}c@{}} avg cell size \\ (\SI{}{\micro\meter}) \end{tabular} & \begin{tabular}{@{}c@{}} avg mito size \\ (\SI{}{\micro\meter}) \end{tabular} & \begin{tabular}{@{}c@{}} avg mito \\ roundness \end{tabular} & \begin{tabular}{@{}c@{}} avg \#of \\ mito per cell \end{tabular}  & $\chi^2$ cell size & $\chi^2$ mito size & \begin{tabular}{@{}c@{}} $\chi^2$ mito \\ roundness \end{tabular} \\
  \hline \hline
  Training   & 0.286 & 0.291 & 0.843 & 0.067 & 0     & 0     & 0     \\
  Non-Param  & \textbf{0.283} & 0.270 & 0.797 & 0.074 & 0.054 & \textbf{0.094} & 0.103 \\
  \hline
  DSGAN   & 0.310 & 0.272 & 0.819 & \textbf{0.068} & \textbf{0.050} & 0.130 & 0.109 \\
  SGAN    & 0.311 & \textbf{0.288} & \textbf{0.839} & 0.073 & 0.057 & 0.112 & 0.122 \\
  \hline
  Joint   & 0.272 & 0.275 & 0.779 & 0.060 & 0.058 & 0.129 & \textbf{0.095} \\
  Unsup   & 0.208 & 0.231 & 0.646 & 0.017 & 0.158 & 0.424 & 0.270 \\
  \hline
  \end{tabular}}
  \caption{Global statistics of multiple cells. We report the numbers of average cell size, average mitochondria size, average mitochondria roundness, average number of mitochondria per cell, and chi-squared distances of their distributions between ground-truth and synthetic data. Numbers suggest SGAN/DSGAN captures global statistics of cellular structures.}
  \label{tab:stat}
\end{table*}
First, we present evaluation results on labels. For shape classifiers, example single cell labels are shown in Fig.~\ref{fig:single} (more in Supplementary 5). From Fig.~\ref{fig:acc}-a, we conclude that SGAN/DSGAN and JointGAN outperform the parametric and unsupervised baselines, which is confirmed by user studies shown in Fig.~\ref{fig:user}-a. A qualitative visualization of the shape features is shown in Fig~\ref{fig:tsne}-a. Moreover, SGAN and DSGAN can recover the global statistics of cells as well (Table~\ref{tab:stat}). Not surprisingly, shape features and global statistics extracted from the trivial non-parametric baseline model look perfect, however, the total number of different images that can be generated (which is also referred to as the support size of the generated distribution) is largely confined by the size of the dataset.

\noindent {\bf Support size:} To address this limitation, we estimate the support size of the generated distribution induced by our model by computing the number of samples that need to be generated before encountering duplicates (the ``Birthday Paradox'' test, as proposed in past work~\cite{arora2017gans}). Our model is able to produce much more diverse samples (please see Supplement 4 for details).





\noindent {\bf Network ablation study:} As discussed, our proposed methods can produce accurate labels, which is achieved by two architectural modifications: (1) fully-convolutional generation, and (2) multi-scale discrimination. To verify their effectiveness, we conduct ablation studies, and particularly, compare FCGAN with three baselines: {\em Non-FC}, where $G_y$ is not fully-convolutional; {\em Non-MS}, where $D_y$ only contains a single discriminator; vanilla {\em DCGAN}, whose results are not shown because of poor qualities (cannot extract single cells from synthesized labels).  Quantitatively, Fig.~\ref{fig:acc}-b illustrates that FCGAN outperforms baselines by a large margin. A t-SNE visualization is shown in Fig.~\ref{fig:tsne}-b. Qualitatively, as shown in Fig.~\ref{fig:fcgan_ablation}, Non-FC, Non-MS and DCGAN all suffer from mode collapse.

\noindent {\bf Segmentation accuracy:} Following past work, we also evaluate realism of an image by the accuracy of an off-the-shelf segmentation network. We report mean IU and negative log likelihood (or NLL). Particularly, for SGAN and DSGAN, we take their image generators and use them to render images from a fixed set of pre-generated synthetic labels, which is used as ``ground-truth'' for evaluating segmentation accuracy. The reason is that eventually at test time, we follow the same process of rendering synthetic images from generated labels. As shown in Table~\ref{tab:fcnscore}, DSGAN has better segmentation accuracy than SGAN and JointGAN, which is confirmed by user studies in Fig.~\ref{fig:user}-b.
\begin{figure}
  \centering
  \includegraphics[width=0.9\linewidth]{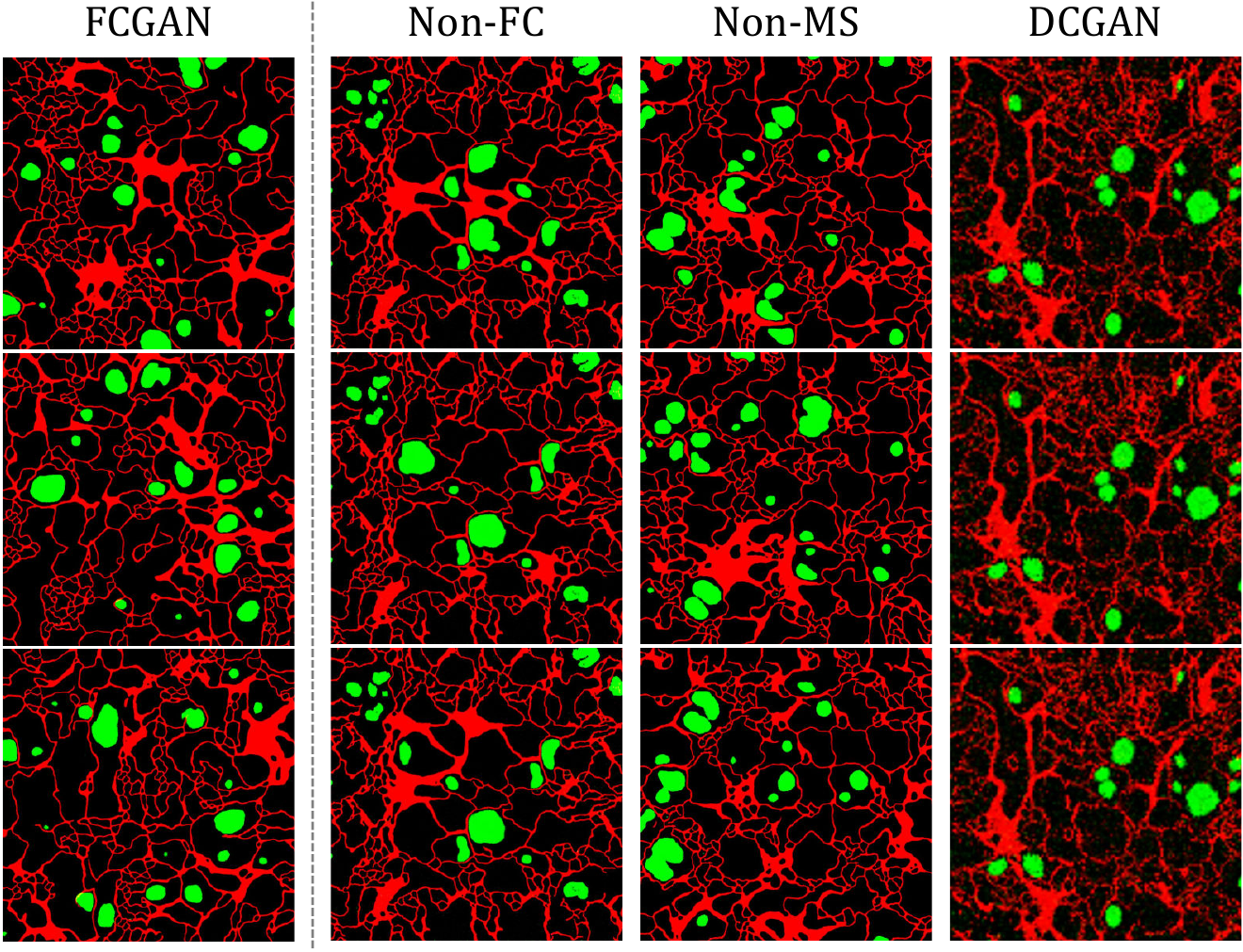}
  \caption{Label synthesis, raw output without label editing. Non-FC, Non-MS and DCGAN all suffer from mode collapse: Non-FC, patterns at four sides are the same across samples, inner patterns are also repetitive; Non-MS, repetitive patterns show at different locations; DCGAN, samples are blurry and almost identical.}
  \label{fig:fcgan_ablation}
\end{figure}
\begin{table}
  \begin{center}
    \resizebox{0.6\linewidth}{!}{
    \begin{tabular}{lcc}
    \hline
    Dataset      & mean IU & NLL               \\
    \hline \hline
    Non-Param      & 88.3\%  & 0.112 $\pm$ 0.006 \\
    \hline
    DSGAN        & 89.3\%  & 0.108 $\pm$ 0.006 \\
    SGAN         & 87.2\%  & 0.132 $\pm$ 0.006 \\
    \hline
    Joint        & 81.8\%  & 0.177 $\pm$ 0.013 \\
    \hline
    \end{tabular}}
  \end{center}
  \caption{Segmentation accuracies for SGAN/DSGAN and baselines. The mean IU and NLL of SGAN/DSGAN both match those of real cell images. Non-FC and Non-MS have high segmentation accuracy due to mode collapse. UnsupervisedGAN is not shown because it does not provide ``ground-truth'' label automatically}
  \label{tab:fcnscore}
\end{table}
\begin{figure}
  \centering
  \includegraphics[scale=0.45]{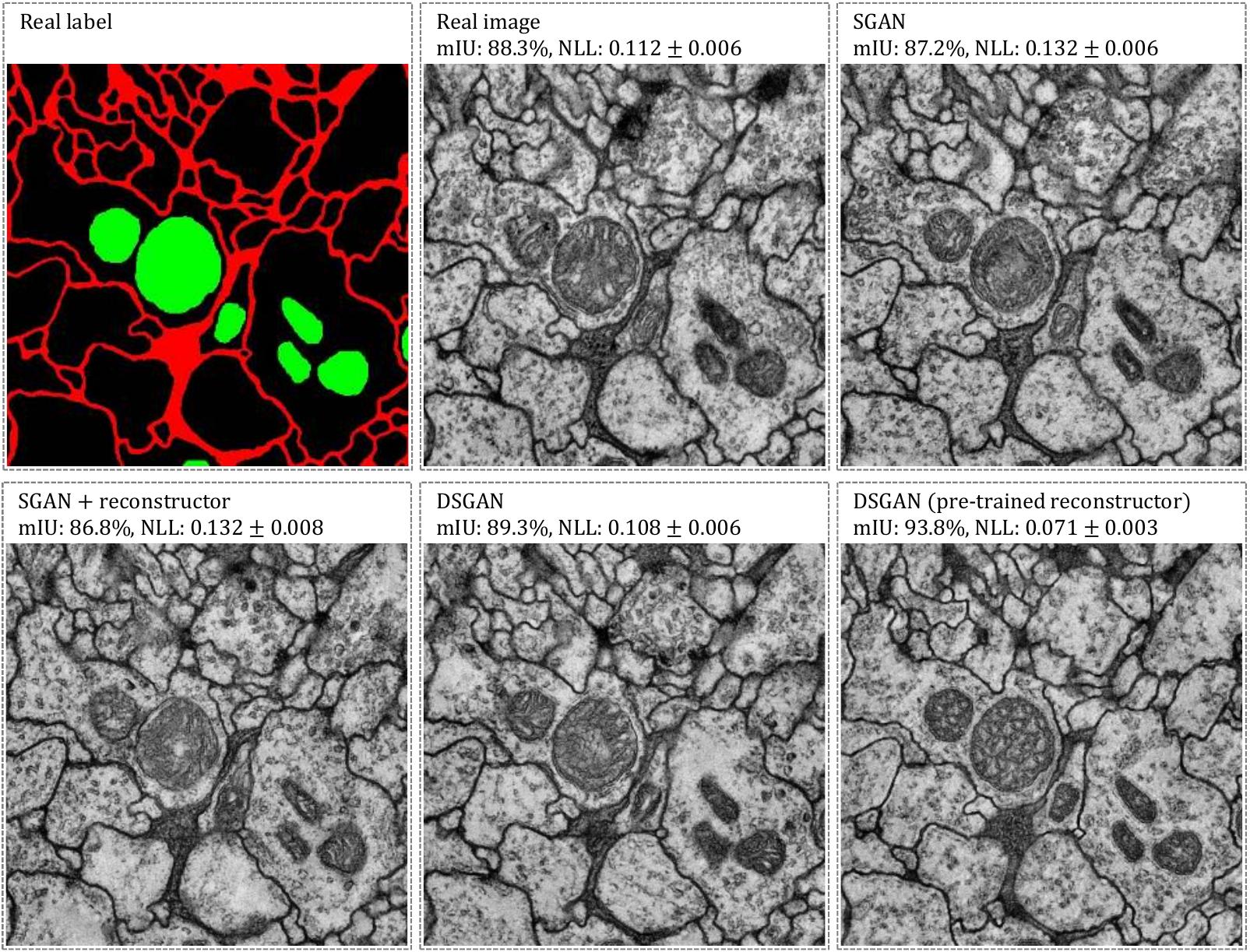}
  \caption{Synthetic image samples and segmentation accuracies of different training approaches. We take $G_x$'s and evaluate segmentation accuracies on a same set of generated labels. $G_x$ of DSGAN yields higher segmentation accuracy but does not show obvious advantage visually. DSGAN with a pre-trained reconstructor achieves the highest score but not in terms of visual inspection.}
  \label{fig:dsgan_sample}
\end{figure}

\noindent {\bf SGAN v.s. DSGAN:} Perhaps it is not surprising that DSGAN performs better than SGAN, since it makes use of an additional reconstruction loss that ensures that generated images will produce segmentation labels that match (or reconstruct) those used to produce the generated images. In theory, one could add such a reconstruction loss to SGAN. However, Fig.~\ref{fig:dsgan_sample} shows that $\text{SGAN}+\text{reconstructor}$ actually has a lower mean IU ($86.8\%$) than vanilla $\text{SGAN}$ ($87.2\%$). Interestingly, because SGAN explicitly factors synthesis into two distinct stages, one can evaluate the second stage module $p(x|y)$ using synthetic labels $\hat{y}$. Under such an evaluation, a reconstruction loss helps ($88.3\%$). In fact, we found one could ``game'' the segmentation metric by using a pre-trained reconstructor, producing a mean IU of $93.8\%$. We found these generated images to be less visually-pleasing, suggesting that the generator tends to overfits to some common patterns recognized by the reconstructor.



\section{Discussion}
In this work, we explore methods towards supervised GAN training, where the generative process is factorized and guided by structural labels. New modifications for both generators and discriminators are also proposed to alleviate mode collapse and allow fully-convolutional generation. Finally, we demonstrate by extensive evaluation that our supervised GANs can synthesize considerably more accurate images than unsupervised baselines.


\noindent {\bf Acknowledgments:} This work was supported in part by National Institutes of Health grant GM103712.  We thank Peiyun Hu and Yang Zou for their helpful comments. We would like to specially thank Chaoyang Wang for insightful discussions.

\noindent {\bf Reproducible Research Archive:} All source code and data used in these studies is available at \url{https://github.com/phymhan/supervised-gan}.
{\small
\bibliographystyle{ieee}
\bibliography{refs}
}

\begin{appendix}
\noindent {\large \bf Supplementary}\par
\vspace{7pt}
In Supplementary, we first give a proof of Theorem~\ref{thm:optimality}. Then we discuss architectural designs for conditional generators. An analysis of the failure when training DSGANs with conditional GAN is then discussed. Next, we give details for evaluating support sizes. Finally, we show additional samples of single-cell labels and label-image pairs.

\section{Proof of Theorem~\ref{thm:optimality}} \label{sup:thm}
\subsection{Optimality Condition for SGAN}
\begin{proof}
{\bf GAN $(G_y, D_y)$ \quad} Denote true data probability density function by $p$ and that induced by generator $G_y$ by $q$. The proof is same as the proof for the optimal discriminator in~\cite{goodfellow2014generative}. Rewrite the value function $V_y$ as an integral form,
\begin{align}
V_y(G_y,D_y) = & \mathbb{E}_{y\sim p_{y}}[\log(D_y(y))] + \nonumber\\
     & \mathbb{E}_{z\sim p_{z}}[\log(1-D_y(G_y(z)))] \nonumber\\
    = & \int p(y) \log(D_y(y)) + \nonumber\\
     & q(y) \log(1-D_y(y)) dy \label{eq:proof_sgan_y}.
\end{align}
\noindent We get the optimal discriminator $D^*_y$ by applying the Euler-Lagrange equation,
\begin{align}
D^*_y(y) = \frac{p(y)}{p(y)+q(y)}.
\end{align}
\noindent Plug $D^*_y$ in $V_y$ yields,
\begin{align}
C_y(G_y) &= V_y(G_y, D^*_y) \nonumber\\
  &= -2\log(2) + 2\cdot JSD(p(y)||q(y)),
\end{align}
\noindent where $JSD$ is the Jensen-Shannon divergence. Since $JSD \geq 0$ and reaches its minimum if and only if $p(y)=q(y)$, $q(y)$ recovers the true label distribution $p(y)$ when $G_y$ and $D_y$ are trained optimally.

\noindent {\bf cGAN $(G_x, D_x)$ \quad} Similarly, rewrite $V_x$ as an integral form,
\begin{align}
V_x(G_x,D_x) = & \mathbb{E}_{x,y\sim p_{x,y}}[\log(D_x(x,y))] + \nonumber\\
    & \mathbb{E}_{x,y\sim p_{x,y}}[\log(1-D_x(G_x(y),y))] \nonumber\\
    = & \iint p(x, y) \log(D_x(x, y)) + \nonumber\\
     & p(y) q(x|y) \log(1-D_x(x, y)) dx dy \label{eq:proof_sgan_x}.
\end{align}
\noindent We get the optimal discriminator $D^*_x = \frac{p(x|y)}{p(x|y)+q(x|y)}$. Minimizing $V_x$ w.r.t. $G_x$ is minimizing the expectation of Jensen-Shannon divergence between real and fake conditional distributions,
\begin{align}
C_x(G_x) &= V_x(G_x, D^*_x) \\
  &= -2\log(2) + 2 \int p(y) JSD(p(x|y)||q(x|y)) dy. \nonumber
\end{align}
\noindent When trained optimally, $q(x|y)=p(x|y)$, and $y \in \{y: p(y)>0\}$.

Training a SGAN with cGAN is equivalent to training a GAN for label $y$ and training a cGAN for image $x$, thus when $D$ and $G$ reach their optimality, the probability density functions induced by generators recover the true joint distribution $q(y) q(x|y) = p(x, y)$.
\end{proof}

\subsection{Optimality Condition for DSGAN}
\begin{proof}
{\bf DSGAN in equation~\ref{eq:dsgan_cgan}\quad} Apply the same tricks to $V_x$,
\begin{align}
V_x(G_x,F_y,D_x) = & \iint p(x, y) \log(D_x(x, y)) + \label{eq:proof_dsgan}\\
    & q(y) q(x|y) \log(1-D_x(x, y)) dx dy \nonumber
\end{align}
\noindent and the optimal $D^*_x = \frac{p(y)p(x|y)}{p(y)p(x|y) + q(y)q(x|y)}$. Then we have
\begin{align}
C_x(G_x) &= V_x(G_x, D^*_x) \\
  &= -2\log(2) + 2\cdot JSD(p(y)p(x|y)||q(y)q(x|y)), \nonumber
\end{align}
\noindent and when trained optimally, $q(y)q(x|y) = p(y)p(x|y)$. Given $q(y) = p(y)$, we have $q(x|y) = p(x|y)$.

\noindent {\bf DSGAN in equation~\ref{eq:dsgan_cgan_cycle}\quad} When $F_y$ is trained optimally, i.e. $F^*_y$ perfectly reconstructs true label $y$ and synthetic label $G_y(z)$, the cross-entropy terms (cycle losses) in equation~\ref{eq:dsgan_cgan_cycle} are zero. Then follow the same proof for optimality condition of cGAN,
\begin{align}
C_x(G_x) &= V_x(G_x, F^*_y, D^*_x) \\
  &= -2\log(2) + 2 \int p(y) JSD(p(x|y)||q(x|y)) dy, \nonumber
\end{align}
\noindent we have $q(x|y) = p(x|y)$.
\end{proof}

\section{Conditional Synthesis} \label{sup:crn}
\begin{figure}
\begin{center}
    \includegraphics[scale=0.55]{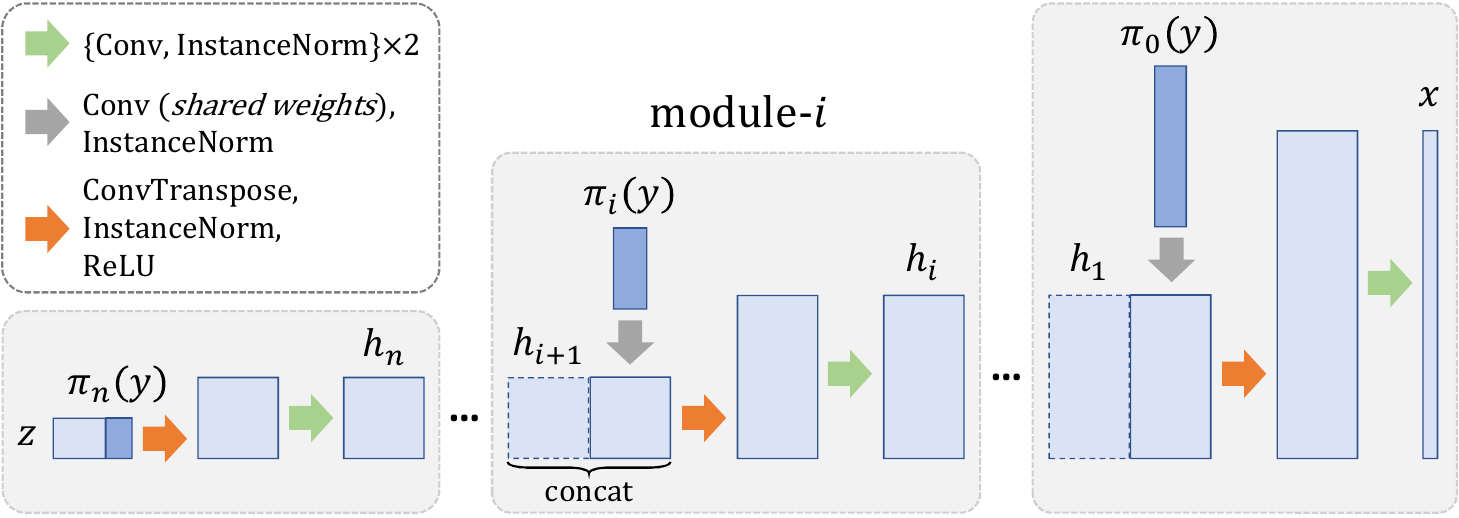}
\end{center}
\caption{Cascaded refinement network for conditional label synthesis. The noise image and down-sampled label go through several similar modules. The modules sequentially refine these intermediate feature $h_i$'s. Down-sampled labels go through a conv layer of which the weights are shared cross all modules (this is to reduce number of parameters as this conv layer only serves the purpose of enlarging dimension of labels to match the dimension of $h_i$'s to prevent the network from ignoring the labels).}
\label{fig:crn}
\end{figure}
\begin{figure}
\begin{center}
    \includegraphics[scale=0.45]{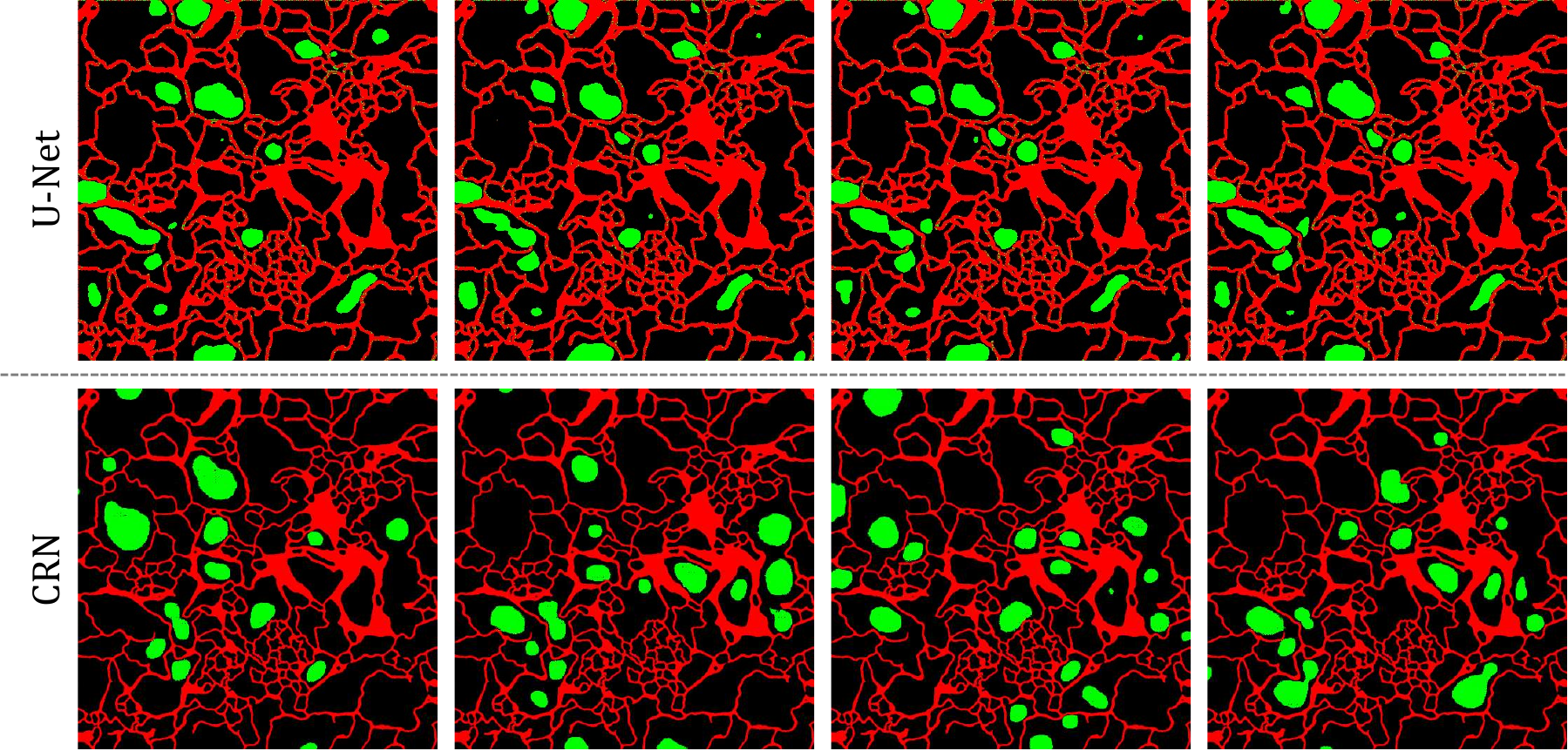}
\end{center}
\caption{Samples for conditional label synthesis. The generators generate different mitochondria (green) given the same membrane layout (red). When using U-Net as generator, the outputs are almost identical. While CRN as generator can generate diverse mitochondria layouts conditioned on the same membrane label.}
\label{fig:crn_mode}
\end{figure}
The architecture of our conditional generator is given in Fig.~\ref{fig:crn}. As discussed in main text, under some cases such as conditional label synthesis, CRN alleviates mode collapse as compared to using U-Net~\cite{ronneberger2015u} as generator. Examples in Fig.~\ref{fig:crn_mode} show that when training a cGAN synthesizing mitochondria given membrane layouts, CRN gives much more diverse outputs than U-Net.
\section{SGAN v.s. DSGAN} \label{sup:dsgan}
Here we analyze the importance of introducing a reconstructor in DSGAN by a case study. As mentioned before, training $G_x$ on fake label-image pairs using cGAN, i.e. training $D_x$ to discriminate $(G_y(z), G_x(G_y(z)))$ pairs from $(y, x)$ pairs ($x$ and $y$ are real samples), yields poor results. We hypothesize that the discriminator $D_x$ will then focus on the differences between real and synthetic labels rather than correlations between label and images. To verify, as gradients give clue to what neural nets are looking at~\cite{cao2015look}, we visualize the gradient of {\em True} node with respect to image after $D_x$ finishing its inference. We compare the magnitudes of gradients in each channels (membrane $y_1$, mitochondria $y_2$ and image $x$). As shown in Fig.~\ref{fig:dsgan_grad} (b), $D_x$'s attention is focused on membrane labels rather than images. Thus, for DSGAN, we introduce a reconstructor.
\section{Support Size} \label{sup:birthday}
\begin{figure}
  \begin{centering}
      \subfloat[Minimum pairwise distance]{%
      \begin{minipage}{1\linewidth}
      \centering
      \includegraphics[width=0.625\linewidth]{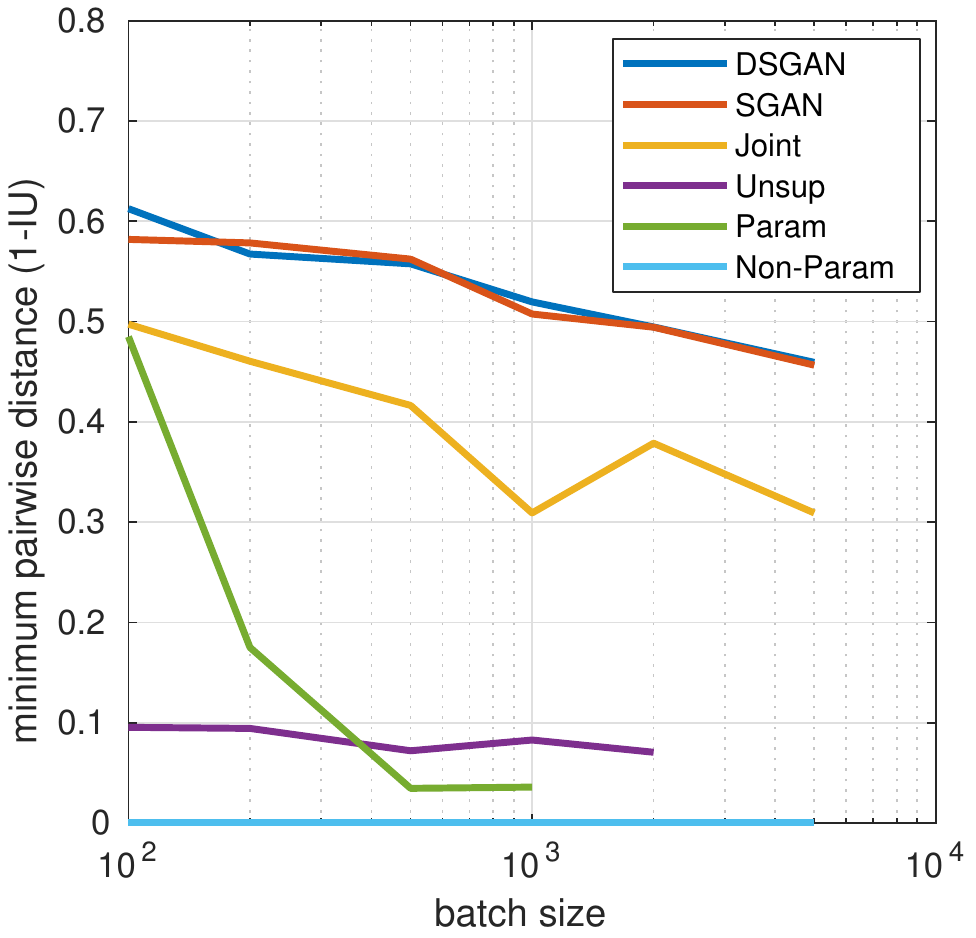}
      \end{minipage}%
      }\par
      \subfloat[Most similar pairs]{%
      \begin{minipage}{1\linewidth}
      \centering
      \includegraphics[width=0.85\linewidth]{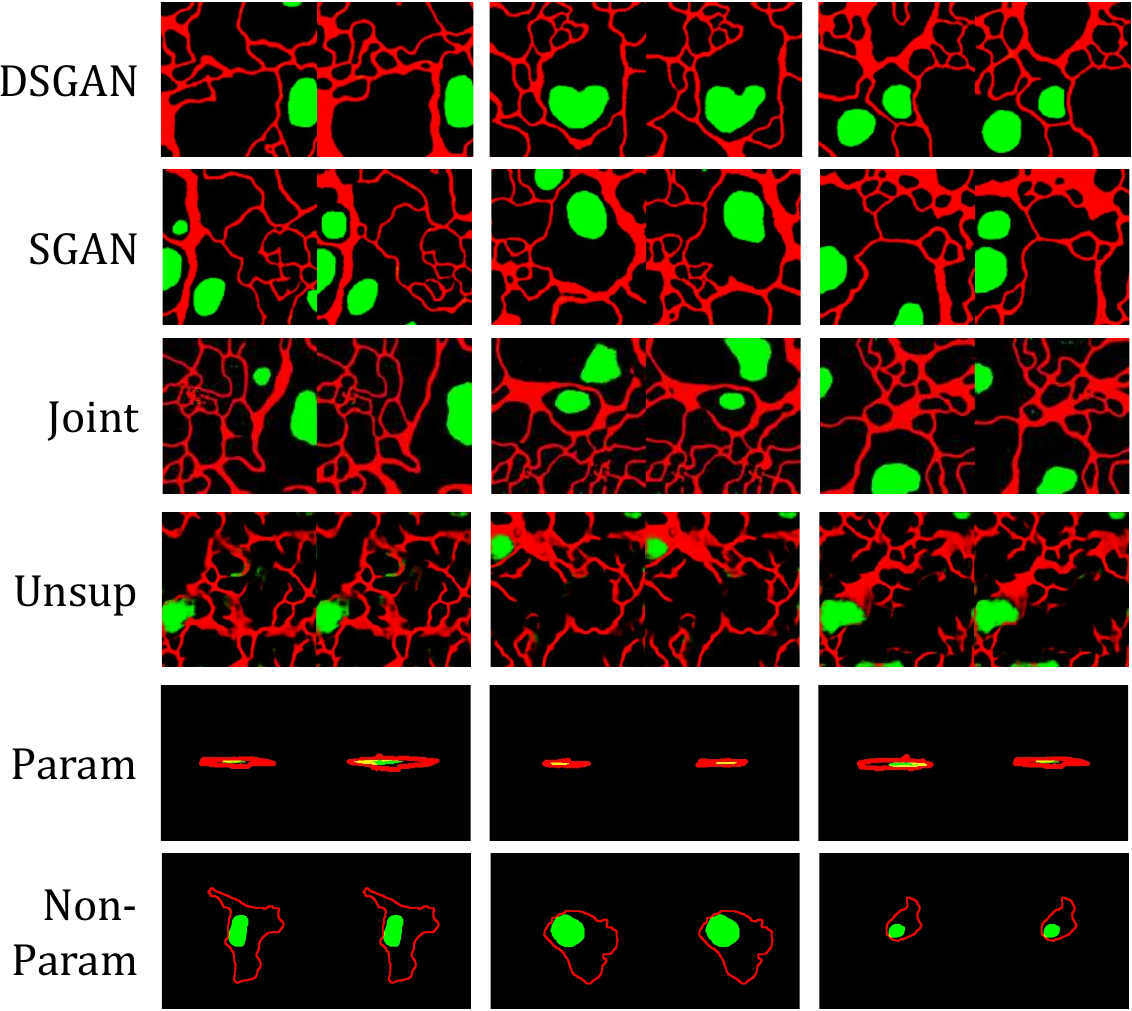}
      \end{minipage}%
      }
      \caption{(a) Plot of minimum pairwise distances over different batch sizes. The distance is defined as $1-\text{IU of the binarized labels}$. (b) Most similar pairs found in the largest batch.}
      \label{fig:birthday}
  \end{centering}
\end{figure}
To evaluate the support size of the generated distribution, we use a technique proposed in~\cite{arora2017gans}. Basically, if we find duplicates in a batch of generated samples of size $s$, the support size is approximately $s^2$. For SGAN, DSGAN and JointGAN, we test samples from a label generator with noise ``images'' of spatial size $2 \times 2$ (an ``atom'' noise image, the output labels are of size $128 \times 128$). The reason is that by enlarging the size of the input noise, the outputs will be the Cartesian product of each ``atom'' output set. The same logic applies to having both membranes and mitochondria. For parametric and non-parametric baseline, we test samples of single cell labels. Particularly, we pick a batch of size $s$ and measure image similarity using the intersection over union (IU) of the binarized labels after a 2-time downsampling (IU greater than $0.9$ is considered as similar). We increase batch-size $s$ from $100$ to $5000$, and run $10$ times for each value of $s$. The visualization of minimum pairwise distances and most similar pairs are shown in Fig.~\ref{fig:birthday}. 

\begin{figure}[h]
\subfloat[SGAN with cGAN]{%
    \begin{minipage}{\linewidth}
    \centering
    \includegraphics[width=\linewidth, valign=b]{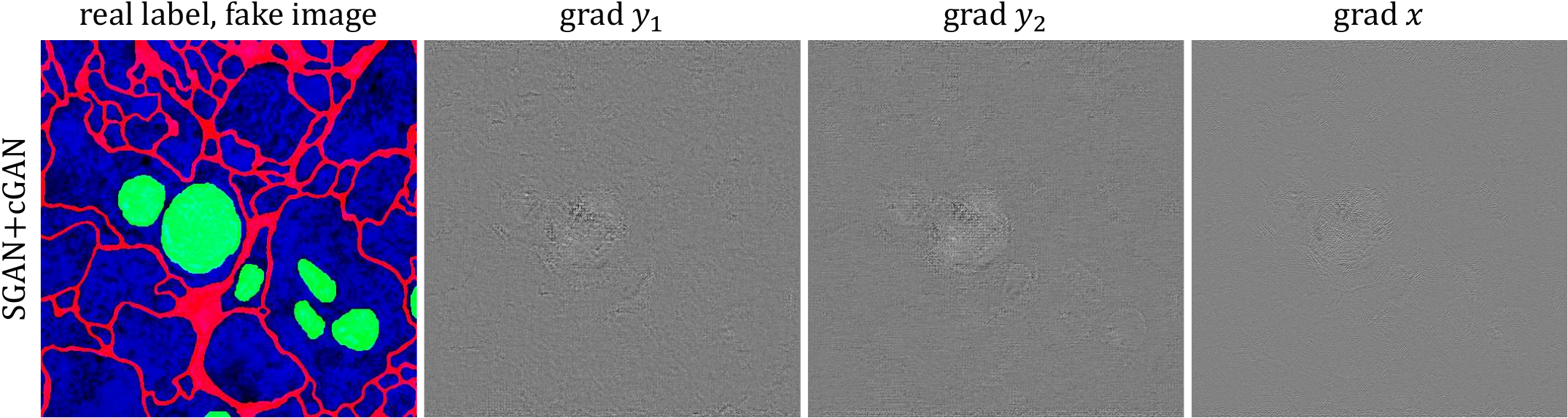}
    \includegraphics[scale=0.4, valign=b]{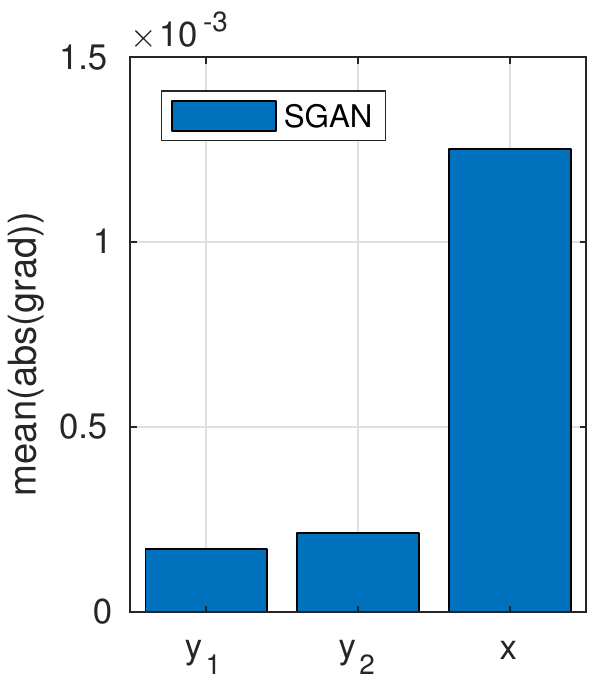}
    \end{minipage}%
}\par
\subfloat[DSGAN with cGAN]{%
    \begin{minipage}{\linewidth}
    \centering
    \includegraphics[width=\linewidth, valign=b]{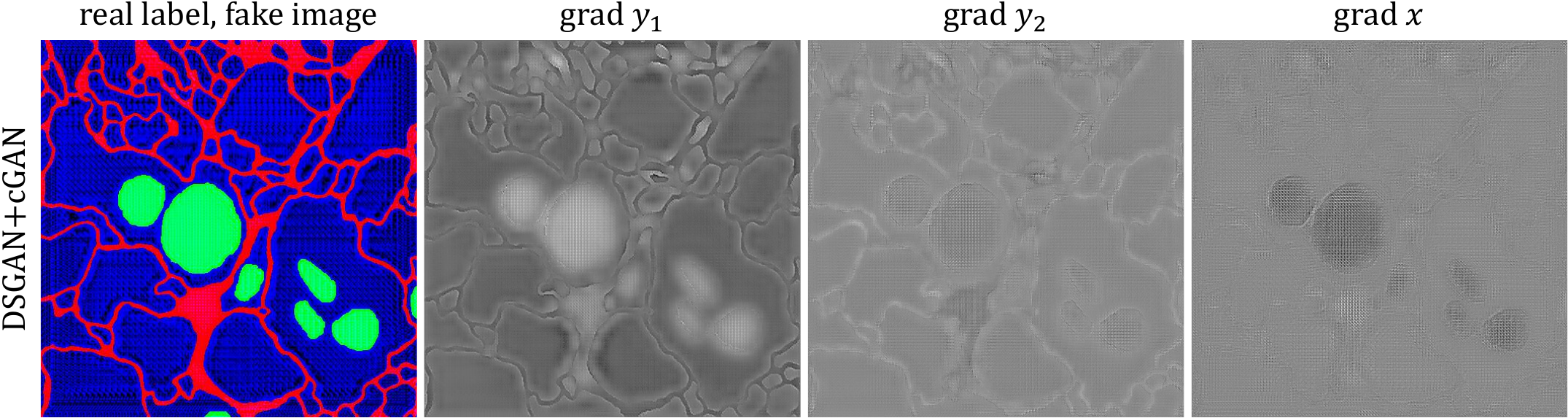}
    \includegraphics[scale=0.4, valign=b]{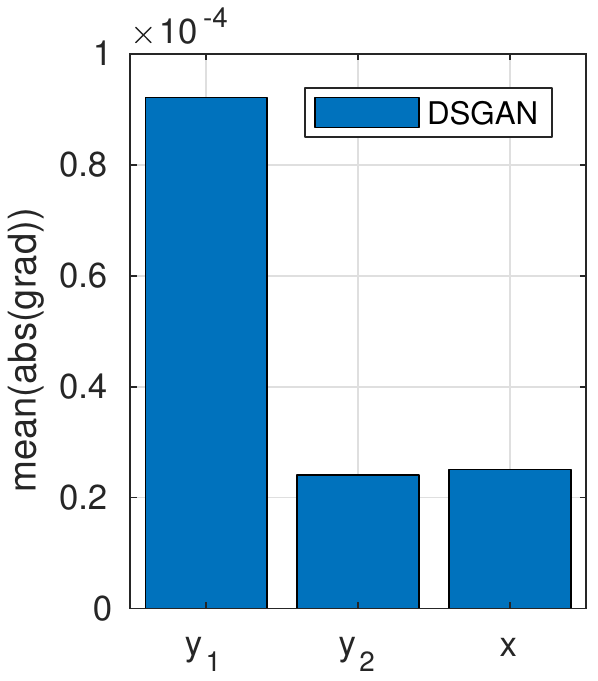}
    \end{minipage}%
}\par
\caption{Gradients of loss of $D_x$ w.r.t. input label-image pairs, when training SGAN/DSGAN with cGAN after 200 epochs. Real label and generated image are shown in a single image in RGB channels, gradient w.r.t. three channels are shown separately. The bar plots show the mean of absolute value of the gradients. For (a), the gradient w.r.t. image channel is larger than gradients w.r.t. labels, suggesting the discriminator learns to focus on translating label to image. For (b), the gradients for labels dominate, suggesting the discriminator is misled by the differences between real and synthetic labels.}
\label{fig:dsgan_grad}
\end{figure}

\end{appendix}

\end{document}